\documentclass{article}

\usepackage[final]{neurips_2022}

\usepackage[utf8]{inputenc} % allow utf-8 input
\usepackage[T1]{fontenc}    % use 8-bit T1 fonts
\usepackage{hyperref}       % hyperlinks
\usepackage{url}            % simple URL typesetting
\usepackage{booktabs}       % professional-quality tables
\usepackage{amsfonts}       % blackboard math symbols
\usepackage{nicefrac}       % compact symbols for 1/2, etc.
\usepackage{microtype}      % microtypography
\usepackage{xcolor}         % colors

\usepackage{mathtools}
\usepackage{amsmath,amssymb, amsthm}
\usepackage{subcaption}
\usepackage{mathrsfs}
\usepackage{enumerate}
\usepackage{multirow}

\newtheorem{theorem}{Theorem}
\newtheorem{example}{Example}
\newtheorem{lemma}{Lemma}

\newtheorem*{theorem*}{Theorem}
\newtheorem*{note*}{Note}
\newtheorem*{lemma*}{Lemma}
\newtheorem*{definition*}{Definition}
\newtheorem*{proposition*}{Proposition}
\newtheorem*{corollary*}{Corollary}
\newtheorem*{result*}{Result}
\newtheorem*{fact*}{Fact}
\newtheorem*{remark*}{Remark}
\newtheorem*{remarks*}{Remarks}

\newcommand{\setdiff}{-}
\newcommand{\UB}{\mathcal{U}\mathcal{B}}

\DeclarePairedDelimiterX\Set[1]\{\}{%
  
  #1
}
\DeclarePairedDelimiterXPP\prob[1]{\mathbb{P}}(){}{
  
  #1
}
\DeclarePairedDelimiterXPP\expt[1]{\mathbb{E}}[]{}{
  
  #1
}

\DeclarePairedDelimiterXPP\exptd[2]{\mathbb{E}^{#2}}[]{}{
  
  #1
}

\DeclarePairedDelimiterXPP\Log[1]{\operatorname{log}}(){}{#1}
\DeclarePairedDelimiterXPP\Exp[1]{\operatorname{exp}}(){}{#1}
\DeclarePairedDelimiterXPP\Gaussian[2]{\mathcal{N}}(){}{#1,~ #2}
\DeclarePairedDelimiter\abs{\lvert}{\rvert}%
\DeclarePairedDelimiter\norm{\lVert}{\rVert}%
\DeclarePairedDelimiter\pbrac{(}{)}%
\DeclarePairedDelimiter\sbrac{[}{]}%
\DeclarePairedDelimiter\cbrac{\{}{\}}%

\makeatletter
\let\oldSet\Set
\def\Set{\@ifstar{\oldSet}{\oldSet*}}
\let\oldprob\prob
\def\prob{\@ifstar{\oldprob}{\oldprob*}}
\let\oldexpt\expt
\def\expt{\@ifstar{\oldexpt}{\oldexpt*}}
\let\oldexptd\exptd
\def\exptd{\@ifstar{\oldexptd}{\oldexptd*}}
\let\oldLog\Log
\def\Log{\@ifstar{\oldLog}{\oldLog*}}
\let\oldExp\Exp
\def\Exp{\@ifstar{\oldExp}{\oldExp*}}
\let\oldabs\abs
\def\abs{\@ifstar{\oldabs}{\oldabs*}}
\let\oldnorm\norm
\def\norm{\@ifstar{\oldnorm}{\oldnorm*}}
\let\oldGaussian\Gaussian
\def\Gaussian{\@ifstar{\oldGaussian}{\oldGaussian*}}
\let\oldpbrac\pbrac
\def\pbrac{\@ifstar{\oldpbrac}{\oldpbrac*}}
\let\oldsbrac\sbrac
\def\sbrac{\@ifstar{\oldsbrac}{\oldsbrac*}}
\let\oldcbrac\cbrac
\def\cbrac{\@ifstar{\oldcbrac}{\oldcbrac*}}
\makeatother
\newcommand{\dprime}{{\prime\prime}}

\DeclareMathOperator*{\argmin}{arg\,min}

\def\equationautorefname~#1\null{Equation (#1)\null}
\def\lemmaautorefname~#1\null{Lemma #1\null}
\def\theoremautorefname~#1\null{Theorem #1\null}
\def\sectionautorefname~#1\null{Section #1\null}
\def\exampleautorefname~#1\null{Example #1\null}
\def\subsectionautorefname~#1\null{Section #1\null}

\title{Asymptotic Properties for Bayesian Neural Network in Besov Space}

\author{%
    Kyeongwon~Lee\\
    Department of Statistics\\
    Seoul National University\\
    Seoul, Republic of Korea 08826 \\
    \texttt{lkw1718@snu.ac.kr} \\
    \And 
    Jaeyong~Lee \\
    Department of Statistics\\
    Seoul National University\\
    Seoul, Republic of Korea 08826 \\
    \texttt{leejyc@gmail.com}
}

\begin{document}

\maketitle

\begin{abstract}
    Neural networks have shown great predictive power when dealing with various unstructured data such as images and natural languages. The Bayesian neural network captures the uncertainty of prediction by putting a prior distribution for the parameter of the model and computing the posterior distribution. In this paper, we show that the Bayesian neural network using spike-and-slab prior has consistency with nearly minimax convergence rate when the true regression function is in the Besov space. Even when the smoothness of the regression function is unknown the same posterior convergence rate holds and thus the spike-and-slab prior is adaptive to the smoothness of the regression function. We also consider the shrinkage prior, which is more feasible than other priors, and show that it has the same convergence rate. In other words, we propose a practical Bayesian neural network with guaranteed asymptotic properties.
\end{abstract}

\section{Introduction}

The neural network is a machine learning technique that can handle complex structures in data by combining linear transformations and non-linear transformations \citep{goodfellow2016deep}. The neural network can show great predictive power when dealing with unstructured data ranging from image data to natural language processing. Moreover, end-to-end learning is possible by grasping the inherent structure of data only with inputs and outputs without explicit intermediate steps. Such properties have played major roles in the rapid development of numerous artificial intelligence applications.

The Bayesian neural network (BNN) provides the uncertainty of the prediction from the probability distribution of the model parameters. Before observing the data, a prior distribution that reflects prior beliefs about the model is put over the parameter space of the neural network model. As observations are added, a posterior distribution is computed by updating the prior distribution via Bayes' rule. \citet{Neal_1996, williams1996computing} showed that a shallow BNN which has only one infinite-width hidden layer converges to a Gaussian process regression model. This means that the shallow BNN can be regarded as a finite approximation of the nonparametric Bayesian regression model. \citet{Lee_Bahri_Novak_Schoenholz_Pennington_Sohl-Dickstein_2017, matthews2018gaussian} extended this result to the deep BNN model.

However, it is difficult to compute the exact posterior distribution of the BNN since the normalizing constant of the posterior distribution is intractable. Therefore, various methods have been proposed to infer BNN through approximate Bayesian computation algorithms. The methods based on the Markov chain Monte Carlo (MCMC) algorithm such as the Gibbs sampler \citep{geman1984stochastic, casella1992explaining} and Hamiltonian Monte Carlo (HMC) \citep{neal2011mcmc} is effective to approximate true posterior. By constructing a Markov chain that has the desired distribution as its stationary distribution, one can obtain a sample of the desired distribution. Recently, scalable MCMC algorithms such as stochastic gradient Langevin dynamics \citep{welling2011bayesian} and stochastic gradient HMC (SGHMC, \citet{chen2014stochastic}) were proposed. The variational inference (VI) is an approximation method that finds the closest variational distribution from the posterior distribution. VI method changes the Bayesian inference problem to an optimization problem. \citet{graves2011practical} proposed to use of a Gaussian distribution as variational distribution for BNN and \citet{kingma2013auto} suggested the \textit{reparameterization trick} for Gaussian variational distribution.
Practically, VI methods for BNN are faster and more stable than MCMC methods. \citet{blundell2015weight} suggested \textit{Bayes by Backprop} algorithm for BNN via variational inference. 
There are  also studies that approximate BNN with the dropout \citep{gal2016dropout} or the Gaussian process \citep{Lee_Bahri_Novak_Schoenholz_Pennington_Sohl-Dickstein_2017}.

\subsection{Related works}

The performance of the neural network is justified by the \textit{universal approximation} capability of the neural network. It is known that even a shallow neural network can approximate arbitrary continuous functions sufficiently precisely in $L^p$-sense \citep{cybenko1989approximation,hornik1991approximation}, and as the network deepens, the expressive power of the neural network gets stronger \citep{delalleau2011shallow,bianchini2014complexity,telgarsky2015representation,telgarsky2016benefits}. \citet{Lee_1999} showed that the BNN has the universal approximation property. 

As a measure of the expressive power of a model, consider the range of function spaces that a model can express. The choice of the activation function can affect the expressive power of a model, and in this paper, we consider the Rectified Linear Unit (ReLU) $ReLU(x) = \max\{0, x\}$ \citep{nair2010rectified}, which is the most popular activation function in practice. \citet{yarotsky2017error} evaluated the approximation error of the neural network using the ReLU activation function (ReLU network) for functions in the H\"{o}lder space. \citet{schmidt2020nonparametric} showed that the regularized least squared estimator performed by the ReLU network in a nonparametric regression problem converges to the true regression function with nearly minimax rate up to logarithmic factors on the H\"older space. \citet{Suzuki_2018b} extended the results in \citet{schmidt2020nonparametric} from the H\"older space to the Besov space which contains the H\"older and Sobolev spaces. \citet{polson2018posterior} proved a Bayesian version of the results in \citet{schmidt2020nonparametric}. They proved the Bayesian ReLU network has posterior consistency, the posterior probability is concentrated on the true parameter as collecting data, when the priors on the parameters of the ReLU network are spike-and-slab priors. \citet{cherief2020convergence} showed that the results in \citet{polson2018posterior} are still valid when variational methods are applied.

All of the above results assume a sparse structure in the neural networks. Sparsity is often assumed in high-dimensional models for asymptotic results.
\citet{castillo2015bayesian} showed theoretical properties including posterior consistency and rates of consistency of the high-dimensional linear regression model under the spike-and-slab prior. \citet{song2017nearly} extended the results in \citet{castillo2015bayesian} to the shrinkage priors which are more feasible than the spike-and-slab prior.

\subsection{Main contribution}

We show that the posterior of the Bayesian ReLU network with a spike-and-slab prior with appropriate width, depth, and sparsity level is consistent with a nearly minimax rate. We extend the result of \citet{Suzuki_2018b} to the Bayesian framework and the results of \citet{polson2018posterior} to the Besov space. As in \citet{polson2018posterior}, we show that the same convergence rate is achieved even when the smoothness parameter of the true regression function is not known. That is, the ReLU network has rich expressive power, and this ability is not constrained even without the smoothness information of the true function. We get similar results for the shrinkage prior.

We try to narrow the gap between the study of the theoretical properties and the practical application of the Bayesian ReLU networks. For instance, Gaussian distribution is frequently considered as a prior distribution in practical Bayesian ReLU networks. In this case, the posterior distribution may not be consistent since the posterior becomes heavy tail distribution over the parameter space \citep{Vladimirova_Verbeek_Mesejo_Arbel_2019}. Approaches to interpreting the Bayesian neural networks as Gaussian Process \citep{Lee_Bahri_Novak_Schoenholz_Pennington_Sohl-Dickstein_2017} avoid this problem by setting the variance of the prior distribution to be proportional to the width of the network. We suggest similar sufficient conditions for the posterior consistency of the Bayesian ReLU networks. Compared to precedent studies, our setting is closer to practical use. The results in \citet{Suzuki_2018b} and \citet{polson2018posterior} require to solve the optimization problem under the $l_0$ constraint or compute posterior from spike-and-slab prior respectively. We suggest BNN model, which is easy to infer and consistent with the (nearly) optimal convergence rate.

The rest of the paper is organized as follows. We set up the problem and introduce the necessary concepts in Section \ref{sec:setup}. We present the obtained theoretical results in Section \ref{sec:main_result} and numerical examples in Section \ref{sec:example}. A summary of the paper and a discussion are in Section \ref{sec:conclusions}.

\section{Preliminaries}\label{sec:setup}

\subsection{Notation}

The floor function $\lfloor x \rfloor$ denotes the largest integer less than or equal to $x \in \mathbb{R}$, the ceiling function $\lceil x \rceil$ denotes the smallest integer greater than or equal to $x \in \mathbb{R}$, and $\norm{a}_p$ is an $l_p$-norm of a real vector $a$. The $u$-th derivative of a $d$-dimensional function $f$ is denoted by $D^u f$ with $u \in \Set{0, 1, \cdots}^{d}$. The covering number $N(\epsilon, A, d)$ is the minimal number of $d$-balls of radius $\epsilon > 0$ necessary to cover a set $A$. We denote the dirac-delta function at $a \in \mathbb{R}$ by $\delta_a$. We use $A \lesssim B$ as shorthand for the inequality $A \leq C B$ for some constant $C>0$ and $A \vee B$ as shorthand for the $\max\{A, B\}$.

\subsection{Model}

Suppose that $n$ input-output observations $\mathbb{D}_n = (X_i, y_i)_{i=1}^n \subset [0,1]^{d} \times \mathbb{R}$ are independent random sample from a regression model 
\begin{equation}\label{eqn:model}
    y_i = f_0(X_i) + \xi_i ~(i=1,2,\cdots, n),
\end{equation}
where $(\xi_i)_{i=1}^n$ is an i.i.d. sequence of Gaussian noises $\Gaussian{0}{\sigma^2}$ with known variance $\sigma^2 > 0$ and $f_0$ is the true regression function belonging to the space $\mathscr{F}$.

In frequentist statistics, the maximum likelihood estimation is used to estimate the function $f_0$. From the machine learning viewpoint, it is equivalent to finding $\hat{f}$ that minimizes the square loss. \citet{Suzuki_2018b} showed that $f_0$ can be estimated with a nearly optimal rate in this setting.

In practice, a sparse regularization term $r(f)$ such as dropout and $l_1$-regularizer is also considered together with a loss function. This is equivalent to estimating the MAP (maximum a posteriori) in the Bayesian modeling. The posterior distribution gives not only predictive values but also the uncertainty of the prediction. Thus, in the Bayesian statistics, it is important to find sufficient conditions for a prior distribution or a regularizer which guarantees the consistency or optimal convergence rate of the posterior.

\subsection{Function spaces}

In this section, we introduce function spaces that statisticians have studied for the \textit{smoothness} of functions. Let $\Omega = [0, 1]^d \subset \mathbb{R}^d$ be the domain of functions we consider. 

The $L^p$-norm of a function is defined as follows:
\begin{equation*}
    \norm{f}_{L^p} := \begin{cases}
        \left(\int_\Omega \abs{f(x)}^p dx\right)^{1/p} & 0 < p < \infty, \\
        \sup_{x \in \Omega} \abs{f(x)} & p = \infty
    \end{cases}
\end{equation*}
where $f$ is a real-valued function defined on $\Omega$. 
The $L^p$ space is the function space consisting of functions with bounded $L^p$-norm and is denoted by $L^p(\Omega) = \Set{f: \norm{f}_{L^p} < \infty}$.

Let $s > 0$ be the smoothness parameter and $m = \lfloor s \rfloor$. The $s$-H\"older norm of function $f$ is defined by
\begin{equation*}
    \norm{f}_{C^s} := \sup_{\norm{u}_1 \leq m} \norm{D^u f}_{L^\infty} + \sup_{\norm{u}_1 = m} \sup_{x, y \in \Omega} \frac{\abs{D^u f(x)  - D^u f(y)}}{\abs{x - y}^{s - m}}.
\end{equation*}
The $s$-H\"older space $C^s(\Omega)$ is defined as a set of functions with bounded $s$-H\"older norm. We call $s$ the smoothness parameter of the $s$-H\"older space. 
Let $k \in \mathbb{N}$ and $1 \leq p \leq \infty$. The Sobolev norm of a function $f$ is defined by 
\begin{equation*}
    \norm{f}_{W^{k, p}} := 
    \begin{cases}
        \left( \sum_{\norm{\alpha}_1 \leq k} \norm{D^\alpha f}_{L^p}^p \right)^{1/p} & 1 \leq p < \infty \\
        \sup_{\norm{\alpha}_1 \leq k} \norm{D^\alpha f}_{L^\infty} & p = \infty. \\ 
    \end{cases}
\end{equation*}
The Sobolev space $W^{k, p}(\Omega)$ is defined by the function space consisting of functions with bounded Sobolev norm.

Note the notion of the smoothness of a function is related to the differentiability of the function. The Besov space extends the concept of smoothness.
Before defining the Besov space, define the $r$-th modulus of smoothness of the function $f$ as follows \citep{Gine_Nickl_2016}: 
\begin{equation}
    w_{r, p}(f, t) = \sup_{\norm{h}_2 \leq t} \norm{\Delta_h^r (f)}_p,
\end{equation}
\begin{equation}
    \Delta_h^r(f)(x) = \begin{cases} \sum_{j=0}^r \binom{r}{j} (-1)^{r-j}f(x+jh) & x \in \Omega, x+ rh \in \Omega, \\
    0 & \text{o.w.} \end{cases}
\end{equation} 
For $0 < p, q \leq \infty,~s >0,~r=\lfloor s \rfloor +1$, define the Besov norm of a function $f$ by
\begin{equation}
    \norm{f}_{B_{p,q}^s}:= \norm{f}_p + \begin{cases} \left(\int_0^\infty \left( t^{-s}w_{r,p}(f,t)\right)^q \frac{dt}{t}\right)^{1/q} & q < \infty, \\ \sup_{t>0} \Set{t^{-s}w_{r,p}(f,t)} & q = \infty.  \end{cases}
    \end{equation}
The Besov space $B_{p,q}^s(\Omega)$ is defined as the set of functions with finite Besov norms \citep{Gine_Nickl_2016}, i.e., 
\begin{equation}
    B_{p,q}^s (\Omega) = \Set{f: \norm{f}_{B_{p,q}^s} < \infty }.
\end{equation}
Note that the Besov spaces can be defined without the differentiability and continuity of functions, and are more general than the H\"{o}lder and Sobolev spaces, which are subspaces of the Besov spaces. The Besov space may contain complicated functions including discontinuous function if $d/p \geq s$. For instance, the Cantor function $f_1$ does not belong to any Sobolev space since it does not have a weak derivative but belongs to a Besov space \citep{Sawano_2018}. In fact, $f_1 \in C^{\log2 / \log 3}([0, 1]) \subset B_{\infty, \infty}^{\log2 / \log 3}([0, 1])$. Also,
\begin{equation}\label{eqn:ex_function2}
f_2(x) = \begin{cases} 1/\log (x/2) , & 0 < x \leq 1 \\ 0, & x=0 \end{cases}
\end{equation}
is not in any H\"older space. Note when $p<2$, the smoothness of the Besov space is inhomogeneous \citep{Suzuki_2018b}. \citet{donoho1994ideal} suggested spatially variable functions including Blocks
\begin{equation}\label{eqn:ex_function_blocks}
    \begin{gathered}
        f_3(x) = \sum_j h_j K(x-x_j),~K(x) = (1 + \text{sgn}(x))/2, \\
        (x_j) = (0.1, 0.13, 0.15, 0.23, 0.25, 0.40, 0.44, 0.65, 0.76, 0.78, 0.81), \\
        (h_j) = (4, -5, 3, -4, 5, -4.2, 2.1, 4.3, -3.1, 2.1, -4.2),
    \end{gathered}
\end{equation}
where $\text{sgn}(x)$ is a function that extracts the sign of a real number $x$, and HeaviSine
\begin{equation}\label{eqn:ex_function_heavisine}
    f_4(x) = 4 \sin (4\pi x) - \text{sgn}(x-0.3) - \text{sgn}(0.72 - x).
\end{equation}
Since $f_2,~f_3$ and $f_4$ are functions of bounded variation, in $B_{1,\infty}^1([0, 1])$ \citep{peetre1976new,Suzuki_2018b}. We plot the functions $f_1,f_2,f_3$ and $f_4$ in Figure \ref{fig:ex_function}.

\begin{figure}[!ht]
    \centering
    \caption{Example functions in the Besov spaces}\label{fig:ex_function}
    \begin{subfigure}[b]{0.4\linewidth}
        \centering
        \includegraphics[width=\linewidth]{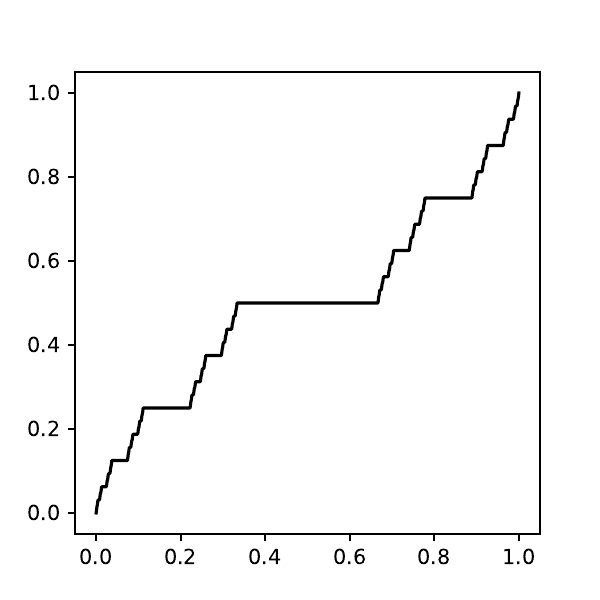}
        \caption{$f_1(x)$ on $[0, 1]$}
        \label{fig:ex_function1}
    \end{subfigure}
    \begin{subfigure}[b]{0.4\linewidth}
        \centering
        \includegraphics[width=\linewidth]{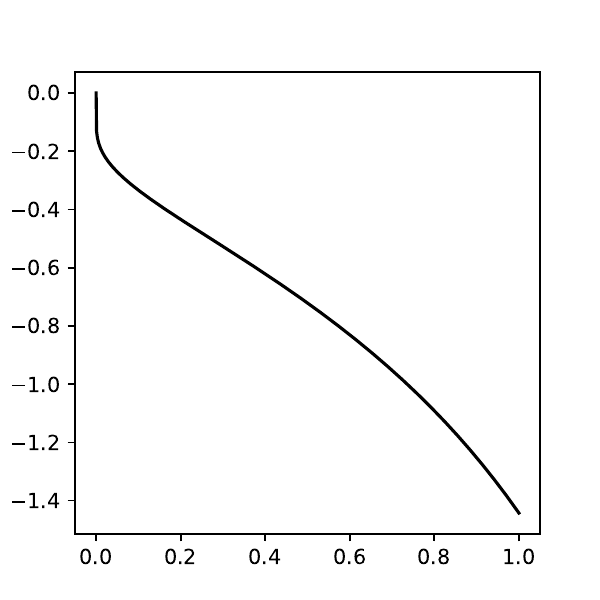}
        \caption{$f_2(x)$ on $[0, 1]$}
        \label{fig:ex_function2}
    \end{subfigure}
    \begin{subfigure}[b]{0.4\linewidth}
        \centering
        \includegraphics[width=\linewidth]{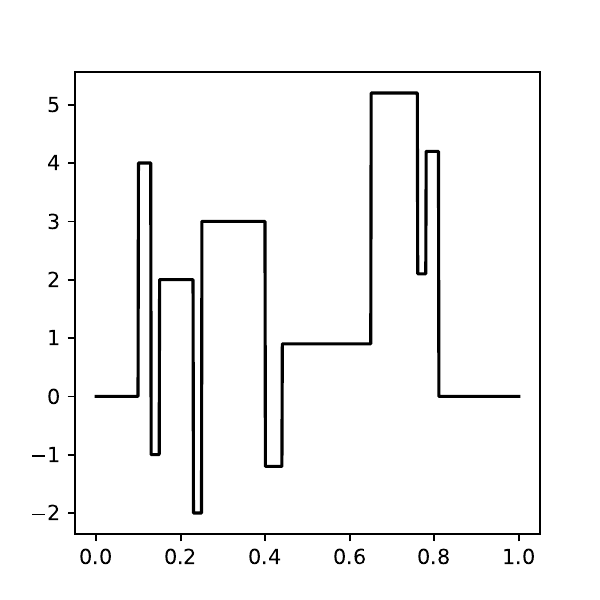}
        \caption{$f_3(x)$ on $[0, 1]$}
        \label{fig:ex_function3}
    \end{subfigure}
    \begin{subfigure}[b]{0.4\linewidth}
        \centering
        \includegraphics[width=\linewidth]{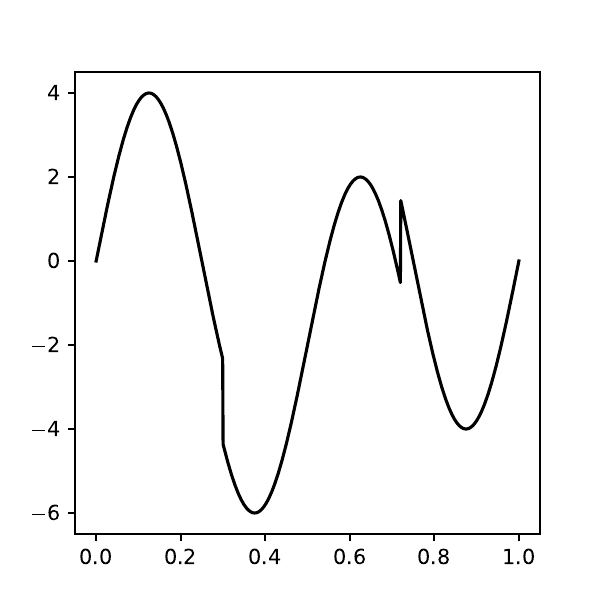}
        \caption{$f_4(x)$ on $[0, 1]$}
        \label{fig:ex_function4}
    \end{subfigure}
\end{figure}

\subsection{Sparse neural network}

Define (deep) neural network space generated by a parameter space $\Theta$ as follows:
\begin{equation}
    \begin{aligned}
        \Phi(\Theta) = \bigg\{
            &f_\theta(x) = (W^{(L+1)}(\cdot) + b^{(L+1)}) \circ \zeta \circ \cdots \circ \zeta \circ (W^{(1)}x + b^{(1)}): \\
            &\theta = (W^{(1)}, b^{(1)}, \cdots, W^{(L+1)}, b^{(L+1)}) \in \Theta \bigg\}.
    \end{aligned}
\end{equation}
where $\zeta$ is an activation function. In this paper, we consider $\zeta(x) = ReLU(x)$ and the following $B$-bounded and $S$-sparse parameter space \citep{Suzuki_2018b}
\begin{equation}
    \begin{aligned}
        \Theta(L, W, S, B):= \bigg\{
            &W^{(l)} \in \mathbb{R}^{p_{l-1} \times p_l}, b^{(l)} \in \mathbb{R}^{p_l}, \\
            &p_l = W, l=1,2\cdots,L,~ p_0=d, p_{L+1}=1, \\
            &\norm{\theta}_0 = \sum_{l=1}^{L+1} \left( \norm{W^{(l)}}_0 + \norm{b^{(l)}}_0 \right) \leq S, \\
            &\norm{\theta}_{\infty} = \max_l \left\{ \norm{W^{(l)}}_{\infty} \vee \norm{b^{(l)}}_{\infty}  \right\} \leq B
            \bigg\}
    \end{aligned}
\end{equation}
and denote $S$-sparse neural network space with $B$-bounded parameters $\Phi(\Theta(L, W, S, B))$ as $\Phi(L, W, S, B)$. We can define the entire neural network model space as follows
\begin{equation}
    \Phi = \bigcup_{L=1}^\infty \bigcup_{W=1}^\infty \bigcup_{S=0}^\infty \bigcup_{B=0}^\infty \Phi(L, W, S, B).
\end{equation}

\section{Main results}\label{sec:main_result}

Let $\UB(F) = \Set{f \in L^{\infty}([0,1]^d): \abs{f(x)} \leq F,~\forall x \in [0, 1]^d}$ be an uniformly bounded family of functions with bound $F >0$.
Consider $\mathscr{F} = B_{p,q}^s([0, 1]^d) \cap \UB(F)$. For each $n \in \mathbb{N}$, and $f \in \mathcal{F}$, let $P_f^{(n)}$ be the probability measure of the data $\mathbb{D}_n = (X_i, y_i)_{i=1}^n$ with a density function $p_f^{(n)}$ with respect to the measure $\mu^{(n)}$ and $\Pi$ be a prior distribution of $f$. The posterior distribution is given by
\begin{equation}
    \Pi(A|\mathbb{D}_n) = \frac{\int_A p_f^{(n)} (\mathbb{D}_n) d\Pi(f)}{\int_{\mathcal{F}} p_f^{(n)} (\mathbb{D}_n) d\Pi(f)},
\end{equation}
for $A \subset \Theta$. Let $$\norm{f}_n = \left( \frac{1}{n} \sum_{i=1}^n \left(f(X_i)\right)^2 \right)^{1/2}$$ be the norm on empirical measure $\mathbb{P}_n^X = \frac{1}{n}\sum_{i=1}^n \delta_{X_i}$ and $$\norm{f}_{L^2(P_X)} = \left( \int \left(f(X)\right)^2 dP_X \right)^{1/2},$$
where $P_X$ is a marginal distribution of $X$. Suppose that $P_X$ has a bounded density function $p_X(x)$ on $[0,1]^d$. Without loss of generality, we assume that $0 \leq p_X(x) \leq R \leq 2$.

% Main Assumption!
Suppose that $0<F<\infty,~0 < p, q \leq \infty,~\omega:=d(1/p - 1/2)_+ < s< \infty$ and set $\nu = (s-\omega)/(2\omega)$. Assume that $m \in \mathbb{N}$ satisfies $0 < s < \min\{m, m-1+1/p\}$. Let $N_n = \lceil n^{d/(2s+d)} \rceil,~W_0 = 6dm(m+2) + 2d$ and
\begin{equation}
    \begin{gathered}
        L_n = L(N_n), \quad W_n = N_n W_0, \\
        S_n = (L_n-1)W_0^2 N_n + N_n, \quad B_n=O(N_n^{\Xi}),
    \end{gathered}
\end{equation}
where $c_{(d,m)}=\left(1 + 2de \frac{(2e)^m}{\sqrt{m}}\right)^{-1},~L(N_n) = 3 + 2\lceil \log_2\left( \frac{3^{d \vee m}}{\tau_n c_{(d, m)}} \right) + 5 \rceil \lceil \log_2(d \vee m)\rceil,~\tau_n = N_n^{-s/d - (\nu^{-1}+d^{-1})(d/p-s)_+} (\log N_n)^{-1}$ and $\Xi=(\nu^{-1} + d^{-1})(1 \vee (d/p-s)_+)$.

Note the theorems in this section are similar to the results in \citet{Suzuki_2018b}, which studied the asymptotic properties of the least squared estimator. They showed that for the least squared estimator
\begin{equation}\label{eqn:suzuki_main}
    \begin{gathered}
        \hat{f} = \argmin_{\bar{f}: f \in \Phi(L_n,W_n,S_n,B_n)} \sum_{i=1}^n (y_i - \bar{f}(x_i))^2, \\
        \mathbb{E}_{\mathbb{D}_n}\left[\norm{f_0 - \hat{f}}_{L^2(P_X)}\right] \lesssim n^{-\frac{2s}{2s+d}} (\log n)^3,
    \end{gathered}
\end{equation}
where $\bar{f}$ is the clipping of $f$ defined by $\bar{f} = \min\Set{\max\Set{f, -F}, F}$ for $f_0 \in \UB(F)$ and $\mathbb{E}_{\mathbb{D}_n}[A]$ is an expectation of $A$ with respect to the training data $\mathbb{D}_n$. Note the minimax convergence rate of a Besov function $f_0$ in (\ref{eqn:model}) has a lower bound $n^{-\frac{s}{2s+d}}$ \citep{Gine_Nickl_2016, donoho1998minimax}. \citet{donoho1998minimax} also showed that the minimax rate of a linear estimator has a lower bound, $n^{-\frac{s-(1/p-1/2)_+}{2s+1-2(1/p-1/2)_+}}$ for $d=1$, $s>1/p$ and $1 \leq p,q \leq \infty$ or $s=p=q=1$. Thus, any linear estimator cannot achieve the minimax optimal rate on the Besov space for $p<2$ and is dominated by the neural network model.

\subsection*{Spike-and-slab prior}

This result is an extension of Theorem 6.1 in \citet{polson2018posterior}. They showed that the Bayesian ReLU networks are consistent with nearly minimax rates in the H\"older space. We extend the function space from the H\"older space to the Besov space which contains the H\"older and Sobolev spaces. Assume the following spike-and-slab prior $\pi(\theta)$% as \citet{polson2018posterior} :
\begin{equation}\label{eqn:sparse_prior}
    \begin{gathered}
        \pi(\theta_j |\gamma_j,L,W,S,B)  = \gamma_j \tilde{\pi}(\theta_j|L,W,S,B)  + (1-\gamma_j) \delta_0(\theta_j),\\ 
        \pi(\boldsymbol{\gamma}|L,W,S,B) = \frac{1}{\binom{T}{S}},
    \end{gathered}
\end{equation}
\begin{equation}\label{eqn:known_hyperparams}
    \pi(L=L_n)=\pi(W=W_n)=\pi(S=S_n)=\pi(B=B_n)=1,
\end{equation}
where $\tilde{\pi}(\theta_j|L,W,S,B) = U(\theta_j;[-B,B])$ and $T = \abs{\Theta(L,W,S,B)}$.

\begin{theorem}\label{thm:main}
    Assume model (\ref{eqn:model}) and prior distribution (\ref{eqn:sparse_prior}) and (\ref{eqn:known_hyperparams}). Suppose that $0<F<\infty,~0<p, q \leq \infty$ and $d(1/p - 1/2)_+ < s$. Then the posterior distribution concentrates at the true function with a rate $\epsilon_n = n^{-s/(2s+d)}(\log n)^{3/2}$. That is, 
    $$\Pi(f_\theta \in \Phi \cap \UB(F): \norm{f_\theta - f_0}_n > M_n \epsilon_n | \mathbb{D}_n) \rightarrow 0$$
    in $P_{f_0}^{(n)}$-probability as $n\rightarrow \infty$ for any $M_n \rightarrow \infty$.
\end{theorem}

\begin{proof}
See appendix \ref{subsec:proof_main}.
\end{proof}

As mentioned in \citet{Suzuki_2018b}, the condition $d(1/p - 1/2)_+ < s$ indicates that $f_0 \in B_{p,q}^s([0,1]^d)$ is in $L^2([0,1]^d)$ and can be discontinuous. When $p=q=\infty$, $B_{p,q}^s([0,1]^d) = C^s([0,1]^d)$ and the result by \citet{polson2018posterior} is a special case of \autoref{thm:main}.

\subsection*{Adaptive estimation}

This result is an extension of Theorem 6.2 in \citet{polson2018posterior} to the Besov space. They showed that the Bayesian ReLU networks adapt to smoothness on the H\"older space under the spike-and-slab prior.

\begin{theorem}\label{thm:adaptive_estimation}
    Assume model (\ref{eqn:model}). Let
    \begin{equation}
        \tilde{L}_n({{H}}) = \lceil {{H}} (\log n) \rceil \vee 1,~ \tilde{W}_n({{H}},N) = {{H}} N,~ \tilde{S}_n({{H}},N) = {{H}} N \tilde{L}_n({{H}}),~ \tilde{B}_n({{H}}, N) = N^{{{H}}}.
    \end{equation}
    Consider the following prior
    \begin{equation}\label{eqn:adaptive_prior}
        N \stackrel{d}{\equiv} 1 \vee \lceil {{Z}} / (\log n)^{2} \rceil, \quad \pi_{{{Z}}}\left({{Z}}\right) = \frac{\lambda_N^{{{Z}}}}{{{Z}}!(e^{\lambda_N}-1)} \quad \text{for}~{{Z}}=1,2,\cdots,
    \end{equation}
    and prior distribution (\ref{eqn:sparse_prior}) of $\theta$ given $\tilde{L}_n({{H}}_n),~\tilde{W}_n({{H}}_n,N),~\tilde{S}_n({{H}}_n,N),~\tilde{B}_n({{H}}_n, N)$ on the function space
    \begin{equation}
        \Phi = \bigcup_{n=1}^{\infty} \bigcup_{N=1}^{\infty} \Phi(\tilde{L}_n({{H}}_n),~\tilde{W}_n({{H}}_n,N),~\tilde{S}_n({{H}}_n,N),~\tilde{B}_n({{H}}_n, N))
    \end{equation}
    for any ${{H}}_n \rightarrow \infty$ and $\lambda_N>0$.
    Suppose that $0<F<\infty,~0<p, q \leq \infty$ and $d(1/p-1/2)_+ < s < \min\{m, m-1 + 1/p \}$.
    The posterior distribution concentrates at the true function with a rate $\epsilon_n = n^{-s/(2s+d)}(\log n)^{3/2}$. That is,
    $$\Pi\left(f_\theta \in \Phi \cap \UB(F): \norm{f_\theta - f_0}_n > M_n \epsilon_n | \mathbb{D}_n\right) \rightarrow 0$$
    in $P_{f_0}^{(n)}$-probability as $n\rightarrow \infty$ for any $M_n \rightarrow \infty$.
\end{theorem}

\begin{proof}
    See appendix \ref{subsec:proof_adaptive}.
\end{proof}

In \autoref{thm:main}, the prior distribution depends on the smoothness parameters $p$, $q$, and  $s$. \autoref{thm:adaptive_estimation} shows that, as in \citet{polson2018posterior}, a similar result can be obtained by considering the appropriate prior distribution even if the values of these parameters are unknown. Note it can be easily shown that BNNs do not overfit as in \citet{polson2018posterior}. 

\subsection*{Shrinkage prior}

Asymptotic properties of the spike-and-slab prior can be obtained using \autoref{thm:main}, but its practical use for large neural network models is unsuitable due to the computational cost. The following prior distributions are more practical than a spike-and-slab prior in terms of computation.

The result to be introduced in this section shows that all the aforementioned results still hold even under the shrinkage prior. Assume the following prior
\begin{equation}\label{eqn:shr_prior}
    \pi(\theta|L,W,S,B) = \prod_{j=1}^{T} g(\theta_j|L,W,S,B)
\end{equation}
where $T = \abs{\Theta(L,W,S,B)}$, $g(t) := g(t|L, W, S, B)$ is a symmetric density function and decreasing on $t>0$.

\begin{theorem}\label{thm:shrinkage}
    Assume model (\ref{eqn:model}), prior distribution (\ref{eqn:known_hyperparams}) and (\ref{eqn:shr_prior}). Suppose that $0<F<\infty,~0<p, q \leq \infty$ and $d(1/p-1/2)_+ < s < \min\{m, m-1 + 1/p \}$. Let $\epsilon_n = n^{-s/(2s+d)}(\log n)^{3/2}$ and $g(t)$ be a function such that 
    \begin{equation}\label{eqn:shr_spike}
    \begin{gathered}
        a_n \leq \frac{\epsilon_n}{72 L_n (B_n \vee 1)^{L_n-1} (W_n+1)^{L_n} } \\
        u_n=\int_{[-a_n, a_n]} g(t|L_n,W_n,S_n,B_n) dt \\
    \end{gathered}
    \end{equation}
    satisfies 
    \begin{equation}\label{cond:shr_spike}
        \frac{S_n}{T_n} > 1 - u_n \geq \frac{S_n}{T_n} \eta_n,
    \end{equation}
    \begin{equation}\label{cond:shr_tail}
        - \log g(B_n|L_n,W_n,S_n,B_n) \lesssim (\log n)^2,
    \end{equation}
    continuous on $[-B_n, B_n]$ and
    \begin{equation}\label{cond:shr_support}
        v_n = \int_{[-B_n, B_n]^c} g(t|L_n,W_n,S_n,B_n) dt = o\left(e^{-K_0 n \epsilon_n^2}\right),
    \end{equation}
    where $\eta_n=\exp(-Kn\epsilon_n^2 / S_n)$ for some $K,~K_0>4$.
    The posterior distribution concentrates at the true function with a rate $\epsilon_n$. That is, 
    $$\Pi(f_\theta \in \Phi \cap \UB(F): \norm{f_\theta - f_0}_n > M_n \epsilon_n | \mathbb{D}_n) \rightarrow 0$$
    in $P_{f_0}^{(n)}$-probability as $n\rightarrow \infty$ for any $M_n \rightarrow \infty$.
\end{theorem}

\begin{proof}
    See appendix \ref{subsec:proof_shrinkage}.
\end{proof}

Note that condition (\ref{cond:shr_spike}) is a continuous approximation for the spike part of (\ref{eqn:sparse_prior}). The next condition (\ref{cond:shr_tail}) means that the prior should have enough thick tail to sample true function. The last condition (\ref{cond:shr_support}) restricts the thickness of the tail to prevent a function from being divergent.

Unlike the other priors, the shrinkage prior is relatively straightforward to implement. For example, to implement the model in \autoref{thm:main}, one can consider the MCMC algorithm \citep{sun2022learning} which combines Gibbs sampler, the Metropolis-Hastings algorithm, and the SGHMC algorithm. The shrinkage prior avoids the posterior computation with varying dimensions, and enables feasible computation.

Note that the conditions are sufficient conditions for the neural network model necessary to estimate the function in the worst case. In other words, we may infer true function $f_0$ sufficiently even for a prior distribution that satisfies only some conditions.

\begin{example}[Gaussian prior]\label{ex:gaussian}
    Let $\epsilon_n = n^{-s/(2s+d)}(\log n)^{3/2}$ and $\psi$ be a density functions of the standard Gaussian distribution $\Gaussian{0}{1}$. 
    $$g(t) = \frac{1}{\sigma_{n}} \psi\left(\frac{t}{\sigma_{n}} \right),\quad \sigma_{n}^2 = \frac{B_n^2}{2 K_0 n \epsilon_n^2}$$
    satisfies (\ref{cond:shr_support}).
\end{example}

\begin{example}[Gaussian mixture prior]\label{ex:gaussian_mixture}
    Let $\epsilon_n = n^{-s/(2s+d)}(\log n)^{3/2}$, $\psi$ and $\Psi$ be a density and inverse survival functions of the standard Gaussian distribution $\Gaussian{0}{1}$ respectively.
    $$g(t) = \pi_{1n} \frac{1}{\sigma_{1n}} \psi\left(\frac{t}{\sigma_{1n}} \right) + \pi_{2n} \frac{1}{\sigma_{2n}} \psi\left(\frac{t}{\sigma_{2n}} \right)$$
    with $$\pi_{2n} = \frac{S_n}{T_n},~ \pi_{1n} = 1- \pi_{2n},$$
    $$a_n \leq \frac{\epsilon_n}{72 L_n (B_n \vee 1)^{L_n-1} (W_n+1)^{L_n} },~\sigma_{2n}^2 < \frac{B_n^2}{2 K_0 n \epsilon_n^2}$$
    and
    $$a_n / \Psi\left( \frac{\pi_{2n}}{\pi_{1n}}\left[ \frac{1}{2} \eta_n - \Psi^{-1}\left(\frac{a_n}{\sigma_{2n}}\right) \right] \right) \leq \sigma_{1n} < a_n / \Psi\left( \frac{\pi_{2n}}{\pi_{1n}}\left[ \frac{1}{2} - \Psi^{-1}\left(\frac{a_n}{\sigma_{2n}}\right) \right] \right)$$
    satisfies (\ref{cond:shr_spike}) and (\ref{cond:shr_support}).
\end{example}

\begin{example}[Relaxed spike-and-slab prior]\label{ex:relaxed_ss}
    Let $\epsilon_n = n^{-s/(2s+d)}(\log n)^{3/2}$, $\psi$ and $\Psi$ be a density and inverse survival functions of the standard Gaussian distribution $\Gaussian{0}{1}$ respectively. 
    $$g(t) = \pi_{1n} \frac{1}{\sigma_{1n}} \psi\left(\frac{t}{\sigma_{1n}} \right) + \pi_{2n} U(t; -B_n,~B_n) $$
    with $$\pi_{2n} = \frac{S_n}{T_n},~ \pi_{1n} = 1- \pi_{2n},$$
    $$a_n \leq \frac{\epsilon_n}{72 L_n (B_n \vee 1)^{L_n-1} (W_n+1)^{L_n} },$$
    and
    $$a_n / \Psi\left( \frac{\pi_{2n}}{2\pi_{1n}}\left[ \frac{a_n}{B_n} - \left(1 - \eta_n \right) \right] \right) \leq \sigma_{1n} < a_n / \Psi\left( \frac{\pi_{2n}}{2\pi_{1n}} \frac{a_n}{B_n} \right)$$
    satisfies (\ref{cond:shr_spike}), (\ref{cond:shr_tail}) and (\ref{cond:shr_support}).

\end{example}

\section{Numerical examples}\label{sec:example}

However, these results are still not enough practical to be used. As the scale of the data $n$ increases, the complexity of the neural network to obtain the theoretical optimality increases rapidly. For instance, consider a problem of estimating the Besov functions $f_1,~f_2,~f_3$ and $f_4$. Assume the Gaussian mixture prior distribution for each parameter as in Example \ref{ex:gaussian_mixture},
\begin{equation*}
    \begin{gathered}
    B_n = 10 N_n^{\min\Set{\Xi, 1}},~K_0=5,~\eta_n = \exp(- K_0 n \epsilon_n^2 / S_n) \\
    a_n = \frac{\epsilon_n}{72 L_n (B_n \vee 1)^{L_n-1} (W_n+1)^{L_n} } \\
    \sigma_{2n}^2 = \frac{B_n^2}{2 (K_0 + 1) n \epsilon_n^2},~ \sigma_{1n} = a_n / \Psi\left( \frac{\pi_{2n}}{\pi_{1n}}\left[ \frac{1}{2} \eta_n - \Psi^{-1}\left(\frac{a_n}{\sigma_{2n}}\right) \right] \right).
    \end{gathered}
\end{equation*}

In \autoref{table:f1_networks} and \autoref{table:f2_networks}, the hyperparameters (model parameters) to estimate each function were presented. These values are calculated numerically and may vary depending on the computer environment. Its speed is affected by the complexity of the true regression function. As can be seen from the tables, the smaller variance is almost zero. Considering the numerical precision, it is indistinguishable from spike-and-slab prior \citep{ishwaran2005spike}.

\begin{table}[ht]
    \centering
    \caption{Hyperparameters of BNN to estimate $f_1 \in B_{\infty, \infty}^{\log2 / \log 3}([0, 1])$.}
    \label{table:f1_networks}
    \begin{tabular}{lrrrrrr}
    \toprule
    $n$ & $L$ & $W$ & $\sigma_{1n}$ & $\sigma_{2n}$ & $\pi_{1n}$ & $\pi_{2n}$ \\
    \midrule
    100 & 39 & 400 & $3.747\times10^{-179}$ & $8.443\times10^{-1}$ & $8.719\times10^{-1}$ & $1.280\times10^{-1}$ \\
    1,000 & 43 & 1,100 & $1.505\times10^{-234}$ & $7.597\times10^{-1}$ & $9.534\times10^{-1}$ & $4.651\times10^{-2}$ \\
    10,000 & 45 & 2,950 & $1.508\times10^{-283}$ & $7.954\times10^{-1}$ & $9.826\times10^{-1}$ & $1.733\times10^{-2}$ \\
    \bottomrule
    \end{tabular}
\end{table}

\begin{table}[ht]
    \centering
    \caption{Hyperparameters of BNN to estimate $f_2,~f_3$ and $f_4 \in B^{1}_{1, 1}([0, 1])$.}
    \label{table:f2_networks}
    \begin{tabular}{lrrrrrr}
    \toprule
    $n$ & $L$ & $W$ & $\sigma_{1n}$ & $\sigma_{2n}$ & $\pi_{1n}$ & $\pi_{2n}$ \\
    \midrule
    100 & 41 & 250 & $8.669\times10^{-172}$ & $6.779\times10^{-1}$ & $7.957\times10^{-1}$ & $2.042\times10^{-1}$ \\
    1,000 & 43 & 500 & $1.259\times10^{-205}$ & $5.028\times10^{-1}$ & $8.977\times10^{-1}$ & $1.022\times10^{-1}$ \\
    10,000 & 47 & 1,100 & $1.963\times10^{-256}$ & $4.894\times10^{-1}$ & $9.535\times10^{-1}$ & $4.642\times10^{-2}$ \\
    \bottomrule
    \end{tabular}
\end{table}

In practice, even the smaller models can approximate the true function with a sufficient precision. We present the numerical experiments with the four functions in \autoref{appendix:numerical_results}.

\section{Conclusions}\label{sec:conclusions}

In this paper, we justify the theoretical properties of neural networks that are currently widely used. Specifically, we focused on the asymptotic properties of the Bayesian neural networks and extend results in \citet{polson2018posterior} to the Besov space from the H\"older space. \autoref{thm:main} shows that the Bayesian neural network with a spike-and-slab prior has consistency posterior with a nearly minimax rate at Besov functions with known smoothness. \autoref{thm:shrinkage} extends the result to the shrinkage prior. According to these theorems, the Bayesian neural networks have posterior consistency and (nearly) optimal rates by assuming proper prior. We found there are three sufficient conditions to achieve posterior consistency, (1) enough mass around zero, (2) thick tail to sample true function (3) while not too much. \autoref{thm:adaptive_estimation} shows that even if the smoothness of the function is not known, the same rate can be obtained by considering a prior on the architecture of the neural networks.

The results in this paper are an extension of the results of \citet{Suzuki_2018b} on the Bayesian neural networks also. We suggest a practical method that can achieve the optimal rate while providing the uncertainty of the prediction. Note they showed the theoretical properties of the neural networks in the Besov space via nonparametric regression based on B-spline. In other words, similar results can be obtained using the Bayesian LARK B-spline model \citep{park2021vy}. 

In future works, performance comparison experiments with other models will also be of great help in establishing the properties of the Bayesian neural network. Empirically, we know that there are fields in which decision-tree-based models such as random forest or neural network models perform better than traditional statistical models that find smooth functional relationships such as kernel regression. For these data, it is natural that there would be a non-smooth relationship between the explanatory variable and the response variable. For this reason, it is essential to compare Bayesian neural network models and statistical models for future studies.
% It is necessary to suggest efficient inferential algorithm for the Bayesian neural network models. \citet{song2020extended, sun2022learning} suggested algorithms for extracting MCMC samples from the spike-and-slab prior distribution through the SGHMC algorithm. We expect that if we use the parameters of the pre-trained models as the initial value of the HMC algorithm, it will be possible to obtain a MCMC samples efficiently. This method would be applicable to large models such as ResNet \citep{he2016deep}. We would like to propose inferential algorithms together through follow-up studies.
% Note the conditions in \autoref{thm:shrinkage} such as tail probability depend on the parameters $L_n,~W_n$ and $B_n$. For adaptive estimation like \autoref{thm:adaptive_estimation}, it is necessary to propose a general shrinkage prior distribution that satisfies all conditions even when $L_n,~W_n$ and $B_n$ are varying.

We obtained results under the limited assumption, setting the networks fully connected and using the ReLU activation function. Also, as mentioned in \autoref{sec:example}, the results are not enough practical to be used. In addition, we do not suggest a result for the adaptive version of \autoref{thm:shrinkage} yet. It is a necessary future study to propose a general shrinkage prior distribution that satisfies all conditions even when the parameters of the network are varying. The results of this study are at a basic level compared to the needs for the Bayesian neural networks in the industrial field, such as visual data analysis and natural language processing. It is important to research the theoretical properties of the Bayesian neural networks used in the practical area through ongoing studies. 

\section*{Acknowledgments}

Kyeongwon Lee and Jaeyong Lee were supported by the National Research Foundation of Korea (NRF) grant funded by the Korea government(MSIT) (No. 2018R1A2A3074973 and 2020R1A4A1018207).

\bibliography{../refs}

\begin{thebibliography}{43}
\providecommand{\natexlab}[1]{#1}
\providecommand{\url}[1]{\texttt{#1}}
\expandafter\ifx\csname urlstyle\endcsname\relax
  \providecommand{\doi}[1]{doi: #1}\else
  \providecommand{\doi}{doi: \begingroup \urlstyle{rm}\Url}\fi

\bibitem[Arratia and Gordon(1989)]{arratia1989tutorial}
R.~Arratia and L.~Gordon.
\newblock Tutorial on large deviations for the binomial distribution.
\newblock \emph{Bulletin of mathematical biology}, 51\penalty0 (1):\penalty0
  125--131, 1989.

\bibitem[Bianchini and Scarselli(2014)]{bianchini2014complexity}
M.~Bianchini and F.~Scarselli.
\newblock On the complexity of neural network classifiers: A comparison between
  shallow and deep architectures.
\newblock \emph{IEEE transactions on neural networks and learning systems},
  25\penalty0 (8):\penalty0 1553--1565, 2014.

\bibitem[Bingham et~al.(2019)Bingham, Chen, Jankowiak, Obermeyer, Pradhan,
  Karaletsos, Singh, Szerlip, Horsfall, and Goodman]{bingham2019pyro}
E.~Bingham, J.~P. Chen, M.~Jankowiak, F.~Obermeyer, N.~Pradhan, T.~Karaletsos,
  R.~Singh, P.~A. Szerlip, P.~Horsfall, and N.~D. Goodman.
\newblock Pyro: Deep universal probabilistic programming.
\newblock \emph{J. Mach. Learn. Res.}, 20:\penalty0 28:1--28:6, 2019.
\newblock URL \url{http://jmlr.org/papers/v20/18-403.html}.

\bibitem[Blundell et~al.(2015)Blundell, Cornebise, Kavukcuoglu, and
  Wierstra]{blundell2015weight}
C.~Blundell, J.~Cornebise, K.~Kavukcuoglu, and D.~Wierstra.
\newblock Weight uncertainty in neural network.
\newblock In \emph{International Conference on Machine Learning}, pages
  1613--1622. PMLR, 2015.

\bibitem[Casella and George(1992)]{casella1992explaining}
G.~Casella and E.~I. George.
\newblock Explaining the gibbs sampler.
\newblock \emph{The American Statistician}, 46\penalty0 (3):\penalty0 167--174,
  1992.

\bibitem[Castillo et~al.(2015)Castillo, Schmidt-Hieber, and Van~der
  Vaart]{castillo2015bayesian}
I.~Castillo, J.~Schmidt-Hieber, and A.~Van~der Vaart.
\newblock Bayesian linear regression with sparse priors.
\newblock \emph{The Annals of Statistics}, 43\penalty0 (5):\penalty0
  1986--2018, 2015.

\bibitem[Chen et~al.(2014)Chen, Fox, and Guestrin]{chen2014stochastic}
T.~Chen, E.~Fox, and C.~Guestrin.
\newblock Stochastic gradient hamiltonian monte carlo.
\newblock In \emph{International conference on machine learning}, pages
  1683--1691. PMLR, 2014.

\bibitem[Ch{\'e}rief-Abdellatif(2020)]{cherief2020convergence}
B.-E. Ch{\'e}rief-Abdellatif.
\newblock Convergence rates of variational inference in sparse deep learning.
\newblock In \emph{International Conference on Machine Learning}, pages
  1831--1842. PMLR, 2020.

\bibitem[Cybenko(1989)]{cybenko1989approximation}
G.~Cybenko.
\newblock Approximation by superpositions of a sigmoidal function.
\newblock \emph{Mathematics of control, signals and systems}, 2\penalty0
  (4):\penalty0 303--314, 1989.

\bibitem[Delalleau and Bengio(2011)]{delalleau2011shallow}
O.~Delalleau and Y.~Bengio.
\newblock Shallow vs. deep sum-product networks.
\newblock \emph{Advances in neural information processing systems},
  24:\penalty0 666--674, 2011.

\bibitem[Donoho and Johnstone(1998)]{donoho1998minimax}
D.~L. Donoho and I.~M. Johnstone.
\newblock Minimax estimation via wavelet shrinkage.
\newblock \emph{The annals of Statistics}, 26\penalty0 (3):\penalty0 879--921,
  1998.

\bibitem[Donoho and Johnstone(1994)]{donoho1994ideal}
D.~L. Donoho and J.~M. Johnstone.
\newblock Ideal spatial adaptation by wavelet shrinkage.
\newblock \emph{biometrika}, 81\penalty0 (3):\penalty0 425--455, 1994.

\bibitem[Gal and Ghahramani(2016)]{gal2016dropout}
Y.~Gal and Z.~Ghahramani.
\newblock Dropout as a bayesian approximation: Representing model uncertainty
  in deep learning.
\newblock In \emph{international conference on machine learning}, pages
  1050--1059. PMLR, 2016.

\bibitem[Geman and Geman(1984)]{geman1984stochastic}
S.~Geman and D.~Geman.
\newblock Stochastic relaxation, gibbs distributions, and the bayesian
  restoration of images.
\newblock \emph{IEEE Transactions on pattern analysis and machine
  intelligence}, \penalty0 (6):\penalty0 721--741, 1984.

\bibitem[Ghosal and Van Der~Vaart(2007)]{ghosal2007convergence}
S.~Ghosal and A.~Van Der~Vaart.
\newblock Convergence rates of posterior distributions for noniid observations.
\newblock \emph{The Annals of Statistics}, 35\penalty0 (1):\penalty0 192--223,
  2007.

\bibitem[Gine and Nickl(2016)]{Gine_Nickl_2016}
E.~Gine and R.~Nickl.
\newblock \emph{Mathematical Foundations of Infinite-Dimensional Statistical
  Models}.
\newblock Cambridge University Press, 2016.
\newblock ISBN 978-1-107-33786-2.
\newblock \doi{10.1017/CBO9781107337862}.
\newblock URL \url{http://ebooks.cambridge.org/ref/id/CBO9781107337862}.

\bibitem[Goodfellow et~al.(2016)Goodfellow, Bengio, Courville, and
  Bengio]{goodfellow2016deep}
I.~Goodfellow, Y.~Bengio, A.~Courville, and Y.~Bengio.
\newblock \emph{Deep learning}, volume~1.
\newblock MIT press Cambridge, 2016.

\bibitem[Graves(2011)]{graves2011practical}
A.~Graves.
\newblock Practical variational inference for neural networks.
\newblock \emph{Advances in neural information processing systems}, 24, 2011.

\bibitem[Hoffman et~al.(2014)Hoffman, Gelman, et~al.]{hoffman2014no}
M.~D. Hoffman, A.~Gelman, et~al.
\newblock The no-u-turn sampler: adaptively setting path lengths in hamiltonian
  monte carlo.
\newblock \emph{J. Mach. Learn. Res.}, 15\penalty0 (1):\penalty0 1593--1623,
  2014.

\bibitem[Hornik(1991)]{hornik1991approximation}
K.~Hornik.
\newblock Approximation capabilities of multilayer feedforward networks.
\newblock \emph{Neural networks}, 4\penalty0 (2):\penalty0 251--257, 1991.

\bibitem[Ishwaran and Rao(2005)]{ishwaran2005spike}
H.~Ishwaran and J.~S. Rao.
\newblock Spike and slab variable selection: frequentist and bayesian
  strategies.
\newblock \emph{The Annals of Statistics}, 33\penalty0 (2):\penalty0 730--773,
  2005.

\bibitem[Kingma and Welling(2013)]{kingma2013auto}
D.~P. Kingma and M.~Welling.
\newblock Auto-encoding variational bayes.
\newblock \emph{arXiv preprint arXiv:1312.6114}, 2013.

\bibitem[Lee et~al.(2017)Lee, Bahri, Novak, Schoenholz, Pennington, and
  Sohl-Dickstein]{Lee_Bahri_Novak_Schoenholz_Pennington_Sohl-Dickstein_2017}
J.~Lee, Y.~Bahri, R.~Novak, S.~S. Schoenholz, J.~Pennington, and
  J.~Sohl-Dickstein.
\newblock Deep neural networks as gaussian processes.
\newblock \emph{arXiv preprint arXiv:1711.00165}, 2017.

\bibitem[Lee~III(1999)]{Lee_1999}
H.~K.~H. Lee~III.
\newblock \emph{Model selection and model averaging for neural networks}.
\newblock PhD thesis, Carnegie Mellon University, 1999.

\bibitem[Matthews et~al.(2018)Matthews, Rowland, Hron, Turner, and
  Ghahramani]{matthews2018gaussian}
A.~G. d.~G. Matthews, M.~Rowland, J.~Hron, R.~E. Turner, and Z.~Ghahramani.
\newblock Gaussian process behaviour in wide deep neural networks.
\newblock \emph{arXiv preprint arXiv:1804.11271}, 2018.

\bibitem[Nair and Hinton(2010)]{nair2010rectified}
V.~Nair and G.~E. Hinton.
\newblock Rectified linear units improve restricted boltzmann machines.
\newblock In \emph{Icml}, 2010.

\bibitem[Neal(1996)]{Neal_1996}
R.~M. Neal.
\newblock Priors for infinite networks.
\newblock In \emph{Bayesian Learning for Neural Networks}, pages 29--53.
  Springer, 1996.

\bibitem[Neal et~al.(2011)]{neal2011mcmc}
R.~M. Neal et~al.
\newblock Mcmc using hamiltonian dynamics.
\newblock \emph{Handbook of markov chain monte carlo}, 2\penalty0
  (11):\penalty0 2, 2011.

\bibitem[Park et~al.(2021)Park, Oh, and Lee]{park2021vy}
S.~Park, H.-S. Oh, and J.~Lee.
\newblock L\'{e}vy adaptive b-spline regression via overcomplete systems.
\newblock \emph{arXiv preprint arXiv:2101.12179}, 2021.

\bibitem[Paszke et~al.(2019)Paszke, Gross, Massa, Lerer, Bradbury, Chanan,
  Killeen, Lin, Gimelshein, Antiga, Desmaison, Kopf, Yang, DeVito, Raison,
  Tejani, Chilamkurthy, Steiner, Fang, Bai, and Chintala]{NEURIPS2019_9015}
A.~Paszke, S.~Gross, F.~Massa, A.~Lerer, J.~Bradbury, G.~Chanan, T.~Killeen,
  Z.~Lin, N.~Gimelshein, L.~Antiga, A.~Desmaison, A.~Kopf, E.~Yang, Z.~DeVito,
  M.~Raison, A.~Tejani, S.~Chilamkurthy, B.~Steiner, L.~Fang, J.~Bai, and
  S.~Chintala.
\newblock Pytorch: An imperative style, high-performance deep learning library.
\newblock In H.~Wallach, H.~Larochelle, A.~Beygelzimer, F.~d\textquotesingle
  Alch\'{e}-Buc, E.~Fox, and R.~Garnett, editors, \emph{Advances in Neural
  Information Processing Systems 32}, pages 8024--8035. Curran Associates,
  Inc., 2019.
\newblock URL
  \url{http://papers.neurips.cc/paper/9015-pytorch-an-imperative-style-high-performance-deep-learning-library.pdf}.

\bibitem[Peetre(1976)]{peetre1976new}
J.~Peetre.
\newblock \emph{New thoughts on Besov spaces}.
\newblock Number~1. Mathematics Department, Duke University, 1976.

\bibitem[Polson and Ro{\v{c}}kov{\'a}(2018)]{polson2018posterior}
N.~G. Polson and V.~Ro{\v{c}}kov{\'a}.
\newblock Posterior concentration for sparse deep learning.
\newblock \emph{Advances in Neural Information Processing Systems}, 31, 2018.

\bibitem[Sawano(2018)]{Sawano_2018}
Y.~Sawano.
\newblock \emph{Theory of Besov Spaces}, volume~56 of \emph{Developments in
  Mathematics}.
\newblock Springer Singapore, 2018.
\newblock ISBN 9789811308352.
\newblock \doi{10.1007/978-981-13-0836-9}.
\newblock URL \url{http://link.springer.com/10.1007/978-981-13-0836-9}.

\bibitem[Schmidt-Hieber(2020)]{schmidt2020nonparametric}
J.~Schmidt-Hieber.
\newblock Nonparametric regression using deep neural networks with relu
  activation function.
\newblock \emph{The Annals of Statistics}, 48\penalty0 (4):\penalty0
  1875--1897, 2020.

\bibitem[Song and Liang(2017)]{song2017nearly}
Q.~Song and F.~Liang.
\newblock Nearly optimal bayesian shrinkage for high dimensional regression.
\newblock \emph{arXiv preprint arXiv:1712.08964}, 2017.

\bibitem[Sun et~al.(2022)Sun, Song, and Liang]{sun2022learning}
Y.~Sun, Q.~Song, and F.~Liang.
\newblock Learning sparse deep neural networks with a spike-and-slab prior.
\newblock \emph{Statistics \& Probability Letters}, 180:\penalty0 109246, 2022.

\bibitem[Suzuki(2018)]{Suzuki_2018b}
T.~Suzuki.
\newblock Adaptivity of deep relu network for learning in besov and mixed
  smooth besov spaces: optimal rate and curse of dimensionality.
\newblock In \emph{International Conference on Learning Representations}, 2018.

\bibitem[Telgarsky(2015)]{telgarsky2015representation}
M.~Telgarsky.
\newblock Representation benefits of deep feedforward networks.
\newblock \emph{arXiv preprint arXiv:1509.08101}, 2015.

\bibitem[Telgarsky(2016)]{telgarsky2016benefits}
M.~Telgarsky.
\newblock Benefits of depth in neural networks.
\newblock In \emph{Conference on learning theory}, pages 1517--1539. PMLR,
  2016.

\bibitem[Vladimirova et~al.(2019)Vladimirova, Verbeek, Mesejo, and
  Arbel]{Vladimirova_Verbeek_Mesejo_Arbel_2019}
M.~Vladimirova, J.~Verbeek, P.~Mesejo, and J.~Arbel.
\newblock Understanding priors in bayesian neural networks at the unit level.
\newblock In \emph{International Conference on Machine Learning}, pages
  6458--6467. PMLR, 2019.

\bibitem[Welling and Teh(2011)]{welling2011bayesian}
M.~Welling and Y.~W. Teh.
\newblock Bayesian learning via stochastic gradient langevin dynamics.
\newblock In \emph{Proceedings of the 28th international conference on machine
  learning (ICML-11)}, pages 681--688. Citeseer, 2011.

\bibitem[Williams(1996)]{williams1996computing}
C.~Williams.
\newblock Computing with infinite networks.
\newblock \emph{Advances in neural information processing systems}, 9, 1996.

\bibitem[Yarotsky(2017)]{yarotsky2017error}
D.~Yarotsky.
\newblock Error bounds for approximations with deep relu networks.
\newblock \emph{Neural Networks}, 94:\penalty0 103--114, 2017.

\end{thebibliography}

%%%%%%%%%%%%%%%%%%%%%%%%%%%%%%%%%%%%%%%%%%%%%%%%%%%%%%%%%%%%
\section*{Checklist}

\begin{enumerate}

\item For all authors...
\begin{enumerate}
  \item Do the main claims made in the abstract and introduction accurately reflect the paper's contributions and scope?
    \answerYes{}
  \item Did you describe the limitations of your work?
    \answerYes{See \autoref{sec:conclusions}.}
  \item Did you discuss any potential negative societal impacts of your work?
    \answerYes{}
  \item Have you read the ethics review guidelines and ensured that your paper conforms to them?
    \answerYes{}
\end{enumerate}

\item If you are including theoretical results...
\begin{enumerate}
  \item Did you state the full set of assumptions of all theoretical results?
    \answerYes{See \autoref{sec:main_result}.}
        \item Did you include complete proofs of all theoretical results?
    \answerYes{See \autoref{sec:main_result} and \autoref{appendix:proofs}.}
\end{enumerate}

\item If you ran experiments...
\begin{enumerate}
  \item Did you include the code, data, and instructions needed to reproduce the main experimental results (either in the supplemental material or as a URL)?
    \answerYes{}
  \item Did you specify all the training details (e.g., data splits, hyperparameters, how they were chosen)?
    \answerYes{}
        \item Did you report error bars (e.g., with respect to the random seed after running experiments multiple times)?
    \answerNA{}
        \item Did you include the total amount of compute and the type of resources used (e.g., type of GPUs, internal cluster, or cloud provider)?
    \answerYes{}
\end{enumerate}

\item If you are using existing assets (e.g., code, data, models) or curating/releasing new assets...
\begin{enumerate}
  \item If your work uses existing assets, did you cite the creators?
    \answerYes{See supplementary material.}
  \item Did you mention the license of the assets?
    \answerYes{}
  \item Did you include any new assets either in the supplemental material or as a URL?
    \answerYes{}
  \item Did you discuss whether and how consent was obtained from people whose data you're using/curating?
    \answerNA{}
  \item Did you discuss whether the data you are using/curating contains personally identifiable information or offensive content?
    \answerNA{}
\end{enumerate}

\item If you used crowdsourcing or conducted research with human subjects...
\begin{enumerate}
  \item Did you include the full text of instructions given to participants and screenshots, if applicable?
    \answerNA{}
  \item Did you describe any potential participant risks, with links to Institutional Review Board (IRB) approvals, if applicable?
    \answerNA{}
  \item Did you include the estimated hourly wage paid to participants and the total amount spent on participant compensation?
    \answerNA{}
\end{enumerate}

\end{enumerate}

%%%%%%%%%%%%%%%%%%%%%%%%%%%%%%%%%%%%%%%%%%%%%%%%%%%%%%%%%%%%

\clearpage 

\appendix

\section{Posterior consistency}\label{appendix:posterior_consistency}

\begin{lemma}[Lemma 2 in \citet{ghosal2007convergence}]\label{thm:lem2ghosal2007}
    Suppose that observation $X^{(n)} = (X_1, \cdots, X_n)$ of independent observations $X_i$. Assume that the distribution $P_{\theta, i}$ of the $i$th component $X_i$ has a density $p_{\theta, i}$ relative to a $\sigma$-finite measure $\mu_i$. Let $P_\theta^{(n)} = \bigotimes_{i=1}^n P_{\theta, i}$ be the product measures and average Hellinger distance
    $$d_n^2(\theta_0, \theta) = \frac{1}{n} \sum_{i=1}^n \int (\sqrt{p_{\theta_0,i}} - \sqrt{p_{\theta, i}})^2 d\mu_i.$$
    Then there exists test $\phi_n$ such that $P_{\theta_0}^{(n)} \phi_n \leq e^{-\frac{1}{2}nd_n^2(\theta_0,\theta_1)}$ and $P_{\theta}^{(n)} (1-\phi_n) \leq e^{-\frac{1}{2}nd_n^2(\theta_0,\theta_1)}$ for all $\theta \in \Theta$ such that $d_n(\theta, \theta_1) \leq \frac{1}{18}d_n(\theta_0,\theta_1)$.
\end{lemma}

\begin{proof}
See \citet{ghosal2007convergence}.
\end{proof}

Let define divergences
\begin{gather}
    K(f,g) = \exptd{\log \frac{f(X)}{g(X)}}{f} = \int f \log \frac{f}{g} d\mu, \\
    V_k(f,g) = \exptd{ \abs{\log \frac{f(X)}{g(X)}}^k }{f}, \\
    V_{k,0}(f,g) = \exptd{ \abs{\log \frac{f(X)}{g(X)} - K(f,g)}^k }{f}
\end{gather}
and set
$$B_n^*( \theta_0, \epsilon; k) = \Set{\theta \in \Theta: \frac{1}{n} \sum_{i=1}^n K(P_{\theta_0, i}, P_{\theta, i}) \leq \epsilon^2,~\frac{1}{n} \sum_{i=1}^n V_{k,0}(P_{\theta_0, i}, P_{\theta, i}) \leq C_k \epsilon^k }.$$ Here, the $C_k$ is the constant satisfying $$\expt{\abs{\bar{X}_n - \expt{\bar{X}_n}}^k} \leq C_k n^{-k/2} \frac{1}{n} \sum_{i=1}^n \expt{\abs{X_i}^k}$$for $k\geq 2$. The existence of $C_k$ is guaranteed by Marcinkiewiz-Zygmund inequality. The following lemma gives a sufficient condition for obtaining posterior consistency.

\begin{lemma}[Theorem 4 of \citet{ghosal2007convergence}]\label{thm:thm4ghosal2007}
    Let $P_\theta^{(n)}$ be product measures and $d_n(\theta_0, \theta) = \frac{1}{n} \sum_{i=1}^n \int (\sqrt{p_{\theta_0,i}} - \sqrt{p_{\theta, i}})^2 d\mu_i$. Suppose that for a sequence $\epsilon_n \rightarrow 0$ such that $n\epsilon_n^2$ is bounded away from zero, some $k>1$, all sufficiently large $j$ and sets $\Theta_n \subset \Theta$ which satisfies following conditions:
    \begin{gather}
        \sup_{\epsilon > \epsilon_n} \log N(\epsilon/36, \Set{\theta \in \Theta_n: d_n(\theta, \theta_0) < \epsilon}, d_n) \leq n\epsilon_n^2, \\
        \frac{\Pi(\Theta \setdiff \Theta_n)}{\Pi(B_n^*(\theta_0, \epsilon_n; k))} = o\left(e^{-2n\epsilon_n^2}\right), \\
        \frac{\Pi(\theta \in \Theta_n: j\epsilon_n < d_n(\theta, \theta_0) \leq 2j\epsilon_n)}{\Pi(B_n^*(\theta_0, \epsilon_n; k))} \leq e^{n\epsilon_n^2 j^2 /4}
    \end{gather}
    Then $P_{\theta_0}^{(n)} \left( \Pi(\theta: d_n(\theta, \theta_0) \geq M_n \epsilon_n | \mathbb{D}_n \right) \rightarrow 0$ for any sequence $M_n \rightarrow \infty$.
\end{lemma}

\begin{proof}
See \citet{ghosal2007convergence}.
\end{proof}

In \citet{ghosal2007convergence}, they mentioned that it can be replaced by any other distance $d_n$ for which the conclusion of \autoref{thm:lem2ghosal2007} holds. Moreover, for $k=2$, \autoref{thm:thm4ghosal2007} work with the smaller neighborhood
\begin{equation}
    \overline{B}_n(\theta_0, \epsilon) = \Set{\theta \in \Theta: \frac{1}{n} \sum_{i=1}^n K(P_{\theta_0, i}, P_{\theta, i}) \leq \epsilon^2,~\frac{1}{n} \sum_{i=1}^n V_{2,0}(P_{\theta_0, i}, P_{\theta, i}) \leq \epsilon^2 }
\end{equation}
instead of $B_n^*(\theta_0, \epsilon_n; k)$.

The following is an adapted version of \autoref{thm:thm4ghosal2007}, which was introduced in \citet{polson2018posterior}.

\begin{lemma}\label{lem:consistency}
Assume model (\ref{eqn:model}). Suppose that $\mathcal{F}$ is uniformly bounded. Let $$A_{\epsilon, M} := \Set{f \in \mathcal{F}: \norm{f - f_0}_n \leq M\epsilon}.$$
If there exist $C > 2 / \sigma^2$ and $\mathcal{F}_n \subset \mathcal{F}$ such that 
\begin{gather}
    \sup_{\epsilon > \epsilon_n} \log N(\epsilon/36, A_{\epsilon, 1} \cap \mathcal{F}_n, \norm{\cdot}_n) \leq n\epsilon_n^2, \\
    \Pi(A_{\epsilon_n, 1}) \geq e^{-C n \epsilon_n^2}, \\
    \Pi(\mathcal{F} \setdiff \mathcal{F}_n) = o\left(e^{-(C \sigma^2 +2)n \epsilon_n^2}\right)
\end{gather}
for any $\epsilon_n \rightarrow 0$, $n\epsilon_n^2 \rightarrow \infty$,
$$\Pi(A_{\epsilon_n, M_n}^c | \mathbb{D}_n) \rightarrow 0$$
in $P_{f_0}^{(n)}$-probability as $n\rightarrow \infty$ for any $M_n \rightarrow \infty$.
\end{lemma}

\begin{proof}
    See appendix \ref{subsec:proof_lem:consistency}.
\end{proof}

\section{Consistency of neural network}\label{appendix:consistency_NN}

The upper bound of the covering number of the neural network space is given by as follows.
\begin{lemma}[Lemma 3 in \citet{Suzuki_2018b}]\label{lem:covering}
    $\forall ~\epsilon>0$, 
    \begin{equation}
        \log N(\epsilon, \Phi(L, W, S, B), \norm{\cdot}_{L^\infty})  \leq (S+1) \Log{ 2 \epsilon^{-1} L(B \vee 1 )^{L} (W+1)^{2L}}
    \end{equation}
    for $W,~L \geq 3$.
\end{lemma}

\begin{proof}
    See appendix \ref{subsec:proof_lem:covering}.
\end{proof}

For $L,~W,~S \in \mathbb{N}$ and $B,~a>0$, define a function space 
\begin{equation}
\Phi(L, W, S, B, a) = \left\{ f_{\theta} : \theta \in \Theta(L, W, S, B, a) \right\},
\end{equation}
where 
$$\Theta(L, W, S, B, a) = \left\{\theta: (\theta_i I(\abs{\theta_i} > a))_{i=1}^{T_n} \in \Theta(L, W, S, B)\right\}.$$

\begin{lemma}\label{lem:covering_a}
    $\forall \epsilon \geq 2 a L (B \vee 1)^{L-1} (W+1)^{L}$,
    \begin{equation}
        \log N(\epsilon, \Phi(L, W, S, B, a), \norm{\cdot}_{L^\infty}) \leq (S+1) \Log{2 \epsilon^{-1} L(B \vee 1)^{L} (W+1)^{2L}}
    \end{equation}
    for $W,~L \geq 3$.
\end{lemma}

\begin{proof}
    See appendix \ref{subsec:proof_lem:covering_a}.
\end{proof}

The following lemma shows that for any function $f_0$ in the Besov space, there exist neural networks which is closed enough to $f_0$.
\begin{lemma}[Proposition 1 in \citet{Suzuki_2018b}]\label{lem:approximation}
    Suppose that $0<p, q,r \leq \infty$, $\omega:=d(1/p-1/r)_+ <s<\infty$ and $\nu = (s-\omega)/(2\omega)$. 
    Assume that $N \in \mathbb{N}$ is sufficiently large and $m \in \mathbb{N}$ satisfies $0<s<\min\{m, m-1+1/p\}$. Let $W_0 = 6dm(m+2) + 2d$. Then, 
    \begin{equation}
        \sup_{f_0 \in U(B_{p,q}^s([0,1]^d))}\inf_{f \in \Phi(L,W,S,B)} \norm{f_0 - f}_{L^r} \lesssim N^{-s/d}
    \end{equation}
    for
    \begin{equation}
        \begin{gathered}
            L = 3 + 2\lceil \log_2\left( \frac{3^{d \vee m}}{\tau(N) c_{(d, m)}} \right) + 5 \rceil \lceil \log_2(d \vee m)\rceil, \quad W = NW_0, \\
            S = (L-1)W_0^2 N + N, \quad B = O\left( N^{(\nu^{-1} + d^{-1})(1 \vee (d/p-s)_+)} \right),
        \end{gathered}
    \end{equation}
    where $U(\mathcal{H})$ is the unit ball of a quasi-Banach space $\mathcal{H}$, $c_{(d,m)}=\left(1 + 2de \frac{(2e)^m}{\sqrt{m}}\right)^{-1}$ and $\tau(N) = N^{-s/d - (\nu^{-1}+d^{-1})(d/p-s)_+} (\log N)^{-1}$.
\end{lemma}

\begin{proof}
    See \citet{Suzuki_2018b}.
\end{proof}

\section{Proof of theorems}\label{appendix:proofs}

\subsection{Proof of Lemma \ref{lem:consistency}}\label{subsec:proof_lem:consistency}

\begin{proof}
    Let $\overline{d}_n(f_0, f) = \norm{f_0 - f}_n$. As in \cite{ghosal2007convergence}, page 214, we may use the norm $\overline{d}_n$ instead of the average Hellinger distance $d_n$ and
    \begin{equation}
        K(P_{f_0, i}, P_{f, i}) = \frac{1}{2\sigma^2} \abs{f_0(x_i) - f(x_i)}^2 = \frac{1}{2} V_{2,0}(P_{\theta_0, i}, P_{\theta, i})
    \end{equation}
    for all $i=1,2,\cdots, n$.
    Then,
    \begin{equation}
        \overline{B}_n(f_0, \epsilon) = \Set{f \in \mathcal{F}: \overline{d}_n^2(f_0,f) \leq \sigma^2 \epsilon^2} = A_{\sigma \epsilon_n, 1}
    \end{equation}
    and it is enough to show that 
    \begin{gather}
        \frac{\Pi(\mathcal{F} \setdiff \mathcal{F}_n)}{\Pi(A_{\sigma \epsilon_n, 1})} = o\left(e^{-2n\epsilon_n^2}\right), \\
        \frac{\Pi(A_{2j \epsilon_n, 1} \setdiff A_{j \epsilon_n, 1})}{\Pi(A_{\sigma \epsilon_n, 1})} \leq e^{n\epsilon_n^2 j^2 /4}
    \end{gather}
    for all sufficiently large $j$.
    \begin{equation}
        \frac{\Pi(\mathcal{F} \setdiff \mathcal{F}_n)}{\Pi(A_{\sigma \epsilon_n, 1})} \leq e^{C n\sigma^2\epsilon_n^2} \Pi(\mathcal{F} \setdiff \mathcal{F}_n) =  o\left(e^{-2n \epsilon_n^2}\right)
    \end{equation}
    and 
    \begin{equation}
        \frac{\Pi(A_{\epsilon_n, 2j} \setdiff A_{\epsilon_n, j})}{\Pi(A_{\sigma \epsilon_n, 1})} \leq e^{C n\sigma^2\epsilon_n^2} \Pi(A_{\epsilon_n, 2j} \setdiff A_{\epsilon_n, j}) \leq e^{n\epsilon_n^2 j^2 /4}
    \end{equation}
    for all sufficiently large $j$.
\end{proof}

\subsection{Proof of Lemma \ref{lem:covering}}\label{subsec:proof_lem:covering}

\begin{proof}
    Let denote $f \in \Phi(L, W, S, B)$ as
    $$f(x) = (W^{(L)}\zeta(\cdot) + b^{(L)}) \circ \cdots \circ (W^{(1)}x + b^{(1)}),$$
    \begin{equation}
        \begin{aligned}
            \mathcal{A}_k^+(f)(x) &= \zeta \circ (W^{(k-1)} \zeta(\cdot) + b^{(k-1)}) \circ \cdots \circ (W^{(1)}x + b^{(1)}), \\
            \mathcal{A}_k^-(f)(x) &= (W^{(L)}\zeta(\cdot) + b^{(L)}) \circ \cdots \circ (W^{(k)}x + b^{(k)})
        \end{aligned}
    \end{equation}
    for $k=2,\cdots, L$, and $\mathcal{A}_{L+1}^-(f)(x) = \mathcal{A}_1^+(f)(x) = x$. Then,
    $$f(x) = \mathcal{A}_{k+1}^-(f) \circ (W^{(k)}\cdot + b^{(k)}) \circ \mathcal{A}_k^+(f)(x).$$
    Since $f \in \Phi(L, W, S, B)$,
    \begin{equation}
        \begin{aligned}
            \norm{\mathcal{A}_k^+(f)(x)}_{\infty} &\leq \max_j \norm{W_{j,:}^{(k-1)}}_1 \norm{\mathcal{A}_{k-1}^+(f)(x)}_{\infty}  + \norm{b^{(k-1)}}_{\infty} \\
            &\leq WB\norm{\mathcal{A}_{k-1}^+(f)(x)}_{\infty} + B \\
            &\leq (W+1)(B \vee 1) \norm{\mathcal{A}_{k-1}^+(f)(x)}_{\infty}  \\
            &\leq (W+1)^{k-1}(B \vee 1)^{k-1},
        \end{aligned}
    \end{equation}
    for all $x$, where $A_{j, :}$ is the $j$-th row of the matrix $A$.
    Similarly,
    \begin{equation}
            \abs{\mathcal{A}_k^-(f)(x_1) - \mathcal{A}_k^-(f)(x_2)} \leq (BW)^{L-k+1} \norm{x_1 - x_2}_{\infty}.
    \end{equation}
    Fix $\epsilon > 0$ and $\theta \in \Theta(L,W,S,B)$. For any $\theta^\ast \in \Theta(L,W,S,B)$ which satisfies $\norm{\theta - \theta^\ast}_{\infty} < \epsilon$,
    \begin{equation}
        \begin{aligned}
            &\abs{f_\theta(x) - f_{\theta^\ast}(x)} \\
            &= \abs{\sum_{k=1}^L \mathcal{A}_{k+1}^-(f_{\theta^\ast}) \circ (W^{(k)} \cdot + b^{(k)}) \circ \mathcal{A}_k^+(f_\theta)(x) - \mathcal{A}_{k+1}^-(f_{\theta^\ast}) \circ (W^{(k)^\ast} \cdot + b^{(k)^\ast}) \circ \mathcal{A}_k^+(f_\theta)(x)} \\
            &\leq \sum_{k=1}^L (BW)^{L-k}\norm{(W^{(k)} \cdot + b^{(k)}) \circ \mathcal{A}_k^+(f_\theta)(x) - (W^{(k)^\ast} \cdot + b^{(k)^\ast}) \circ \mathcal{A}_k^+(f_\theta)(x)}_{L^\infty} \\
            &\leq \sum_{k=1}^L (BW)^{L-k} \epsilon \left[W(B \vee 1)^{k-1}(W+1)^{k-1} +1 \right] \\
            &\leq \sum_{k=1}^L (BW)^{L-k} \epsilon (B \vee 1)^{k-1}(W+1)^k \\
            &\leq \epsilon L(B \vee 1)^{L-1} (W+1)^L.
        \end{aligned}
    \end{equation}
    Let $s$ be the number of nonzero components of $\theta$, $s\leq S$. Consider a subspace $\Theta_\theta(L, W, S, B)$ of the parameter space $\Theta(L, W, S, B)$ consists of parameters which have $s$ nonzero component. Choose $\theta_1, \cdots, \theta_M$ from each grid divided with length $\epsilon^\ast = \frac{\epsilon}{L(B \vee 1)^{L-1} (W+1)^L}$ over $\Theta_\theta(L, W, S, B)$. Then $$f_\theta \in \bigcup_{m=1}^M \Set{f: \norm{f - f_{\theta_m}}_{L^\infty} < \epsilon }.$$ Note
    $$T = \abs{\Theta(L,W,S,B)} \leq \sum_{l=1}^L (W+1)W \leq (L+1)(W+1)^2 \leq (W+1)^{L}$$
    for $W,~L \geq 3$, and the number of cases of choose $s$ nonzero components are
    $$\binom{T}{s} = \frac{T(T-1) \cdots (T-s+1)}{s!} \leq T^s \leq (W+1)^{Ls}.$$
    Thus,
    \begin{equation}
        \begin{aligned}
            N(\epsilon, \Phi(L, W, S, B), \norm{\cdot}_{L^\infty}) &\leq \sum_{s^\ast \leq S} \binom{T}{s^\ast} ( 2B \epsilon^{-1} (L(B \vee 1)^{L-1} (W+1)^L )^{s^\ast} \\
            &\leq \sum_{s^\ast \leq S} ( 2 \epsilon^{-1} L(B \vee 1)^{L} (W+1)^{2L})^{s^\ast} \\
            &\leq (2 \epsilon^{-1}L(B \vee 1)^{L} (W+1)^{2L})^{S+1}
        \end{aligned}
    \end{equation}
    and we get desired results by taking the logarithm of both sides.
\end{proof}

\subsection{Proof of Lemma \ref{lem:covering_a}}\label{subsec:proof_lem:covering_a}

\begin{proof}
    Let $\tilde{\theta}(a) = \left(\theta_i I(\abs{\theta_i} > a)\right)$. Then $\norm{\theta - \tilde{\theta}(a)}_\infty \leq a$ and $$\tilde{\Theta}(L, W, S, B, a) = \left\{\tilde{\theta}(a): \theta \in \Theta(L, W, S, B, a)\right\} = \Theta(L, W, S, B).$$
    As in the proof of \autoref{lem:covering}, \begin{equation}
        \norm{f_\theta - f_{\tilde{\theta}(a)}}_{L^\infty} \leq a L (B \vee 1)^{L-1} (W+1)^{L} \leq \epsilon  / 2
    \end{equation}
    for any $f_\theta,~f_{\tilde{\theta}(a)} \in \Phi(L, W, S, B, a, F)$.
    Let $s \leq S$ be the number of nonzero components of $\tilde{\theta}(a)$. Consider a subspace $\tilde{\Theta}_\theta(L, W, S, B, a)$ of $\tilde{\Theta}(L, W, S, B, a)$ consisting of parameters which have $s$ nonzero component. Choose $\theta_1, \cdots, \theta_M$ from each grid divided with length $\epsilon^\ast = \frac{\epsilon /2}{L(B \vee 1)^{L-1} (W+1)^L}$ over $\tilde{\Theta}_\theta(L, W, S, B, a)$. 
    Then $$f_\theta \in \bigcup_{m=1}^M \Set{f: \norm{f - f_{\theta_m}}_{L^\infty} < \epsilon  }$$
    from triangular inequality. We get desired results in a similar way as in the proof of \autoref{lem:covering}.
\end{proof}

\subsection{Proof of Theorem \ref{thm:main}}\label{subsec:proof_main}

\begin{proof}
    Let $\mathcal{F} = \Phi \cap \UB(F)$ and $$A_{\epsilon, M} := \Set{f \in \mathcal{F}: \norm{f - f_0}_n \leq M\epsilon}.$$ 
    From \autoref{lem:consistency}, it is enough to show that there exist a constant $C > 2 / \sigma^2$ and $\mathcal{F}_n \subset \mathcal{F}$ which satisfy
    \begin{enumerate}[(a)]
        \item $\sup_{\epsilon > \epsilon_n} \log N(\epsilon/36, A_{\epsilon, 1} \cap \mathcal{F}_n, \norm{\cdot}_n) \leq n\epsilon_n^2$
        \item $-\log \Pi(A_{\epsilon_n, 1}) \leq C n \epsilon_n^2$
        \item $\Pi(\mathcal{F} \setdiff \mathcal{F}_n) = o\left(e^{-(C \sigma^2 +2)n \epsilon_n^2}\right)$
    \end{enumerate}
    for sufficiently large $n$. Let $\mathcal{F}_n = \Phi(L_n, W_n, S_n, B_n) \cap \UB(F)$. First, (c) is trivial from (\ref{eqn:known_hyperparams}).
    From $$\Set{f \in \mathcal{F}_n : \norm{f}_{L^\infty} \leq \epsilon} \subset \Set{f \in \mathcal{F}_n: \norm{f}_n \leq \epsilon}$$
    and \autoref{lem:covering},
    \begin{equation}
        \begin{aligned}
            &\sup_{\epsilon > \epsilon_n} \log N\left( \frac{\epsilon}{36}, A_{\epsilon, 1} \cap \mathcal{F}_n, \norm{\cdot}_n \right) \\
            &\leq \sup_{\epsilon > \epsilon_n} \log N\left( \frac{\epsilon}{36}, A_{\epsilon, 1} \cap \mathcal{F}_n, \norm{\cdot}_{L^\infty} \right) \\
            &\leq \sup_{\epsilon > \epsilon_n} \log N\left( \frac{\epsilon}{36}, \mathcal{F}_n, \norm{\cdot}_{L^\infty} \right) \\
            &\leq \log N\left( \frac{\epsilon_n}{36}, \mathcal{F}_n, \norm{\cdot}_{L^\infty} \right) \\
            &\leq (S_n+1) \left[ \log L_n + L_n \Log{ (B_n \vee 1) (W_n + 1)^2} - \log \frac{\epsilon_n}{72} \right] \\
            &\lesssim N_n (\log n)^3 \\
            &= n\epsilon_n^2
        \end{aligned}
    \end{equation}
    for sufficiently large $n$. Thus, (a) holds. Here, the last inequality holds from
    \begin{equation}\label{eqn:assymp_rel}
        L_n = O(\log n),~W_n = O(N_n),~ S_n = O(N_n \log n)
    \end{equation}
    Next, from \autoref{lem:approximation}, there is a constant $C_1>0$ and $\hat{f}_n = f_{\hat{\theta}} \in \mathcal{F}$ such that
    \begin{equation}
        \norm{\hat{f}_n - f_0}_{L^2} \leq C_1 \norm{f_0}_{B^s_{p,q}([0, 1]^d)} N_n^{-s/d} \leq \epsilon_n/4.
    \end{equation}
    for sufficiently large $n$. Moreover, by the assumptions and the strong law of large numbers,
    \begin{equation}
        \norm{f - f_0}_n^2 \leq 2 \norm{f - f_0}_{L^2(P_X)}^2 \leq 4 \norm{f - f_0}_{L^2}^2
    \end{equation}
    for sufficiently large $n$ $P_{f_0}^{(n)}$-almost surely. Let $\hat{\gamma}$ and $\hat{\theta}_{\hat{\gamma}}$ be index and value of nonzero components of $\hat{\theta}$ respectively. Let $\Theta(\hat{\gamma}; L_n, W_n, S_n, B_n)$ be a subset of parameter space $\Theta(L_n, W_n, S_n, B_n)$ consists of parameters which have nonzero components at $\hat{\gamma}$ only and $\mathcal{F}_n(\hat{\gamma}) = \Phi(\hat{\gamma}; L_n, W_n, S_n, B_n) \cap \UB(F)$ be an uniformly bounded neural network space generated by $\Theta(\hat{\gamma}; L_n, W_n, S_n, B_n)$. Note
    \begin{equation}
        \begin{aligned}
            \Pi(A_{\epsilon_n, 1} )
            &= \Pi(f \in \mathcal{F}_n: \norm{f - f_0}_n \leq \epsilon_n) \\
            &\geq \Pi(f \in \mathcal{F}_n: \norm{f - f_0}_{L^2} \leq \epsilon_n / 2) \\
            &\geq \Pi\left(f \in \mathcal{F}_n: \norm{f - \hat{f}_n}_{L^2} \leq \epsilon_n / 4\right) \\
            &\geq \Pi\left(f \in \mathcal{F}_n: \norm{f - \hat{f}_n}_{L^\infty} \leq \epsilon_n/4\right) \\
            &\geq \Pi\left(f \in \mathcal{F}_n(\hat{\gamma}): \norm{f - \hat{f}_n}_{L^\infty} \leq \epsilon_n/4\right)
        \end{aligned}
    \end{equation}
    for sufficiently large $n$. As in the proof of \autoref{lem:covering},
    \begin{equation}
        \begin{aligned}
             &\Pi\bigg( f \in \mathcal{F}_n(\hat{\gamma}): \norm{f - \hat{f}_n}_{L^\infty} \leq \epsilon_n/4\bigg) \\
             &\geq \Pi\left(\theta \in \mathbb{R}^{T_n}: \theta_{\hat{\gamma}^c} = 0,~ \norm{\theta}_{\infty} \leq B_n,~ \norm{\hat{\theta} - \theta}_{\infty} \leq \frac{\epsilon_n}{4 (W_n+1)^{L_n} L_n (B_n \vee 1)^{L_n-1}} \right) \\
             &\geq\left( \frac{\epsilon_n}{4 B_n (W_n+1)^{L_n} L_n (B_n \vee 1)^{L_n-1}} \right)^{S_n} \binom{T_n}{S_n}^{-1} \\
             &\geq\left( \frac{\epsilon_n}{4 B_n (W_n+1)^{2L_n} L_n (B_n \vee 1)^{L_n-1}} \right)^{S_n} \\
             &= \Exp{-S_n \Log{\frac{4 B_n (W_n+1)^{2L_n} L_n  (B_n \vee 1)^{L_n-1}}{\epsilon_n} } }
        \end{aligned}
    \end{equation}
    Thus,
    \begin{equation}
        \begin{aligned}
            -\log \Pi(A_{\epsilon_n, 1}) &\leq S_n \Log{\frac{4 B_ n (W_n+1)^{2L_n} L_n  (B_n \vee 1)^{L_n-1}}{\epsilon_n} }  \\
            &\leq S_n \left[L_n  \Log{(W_n+1)^2(B_n \vee 1)} + \log 2 L_n - \log \epsilon_n \right] \\
            &\lesssim N_n (\log n)^3 \\ 
            &= n \epsilon_n^2
        \end{aligned}
    \end{equation}
    for sufficiently large $n$. 
\end{proof}

\subsection{Proof of Theorem \ref{thm:adaptive_estimation}}\label{subsec:proof_adaptive}

\begin{proof}
    Let $\mathcal{F} = \Phi \cap \UB(F)$. From \autoref{lem:consistency}, it is enough to show that there exist a constant $C^\prime>2 / \sigma^2$ and $\mathcal{F}_n \subset \mathcal{F}$ which satisfy
    \begin{enumerate}[(a)]
        \item $\sup_{\epsilon > \epsilon_n} \log N(\epsilon/36, A_{\epsilon, 1} \cap \mathcal{F}_n, \norm{\cdot}_n) \leq n\epsilon_n^2$
        \item $-\log \Pi(A_{\epsilon_n, 1}) \leq C^\prime n \epsilon_n^2$
        \item $\Pi(\mathcal{F} \setdiff \mathcal{F}_n) = o\left(e^{-(C^\prime \sigma^2 +2)n \epsilon_n^2}\right)$
    \end{enumerate}
    for sufficiently large $n$. From (\ref{eqn:assymp_rel}), we can choose
    \begin{equation}
        {{H}}_0 > \sup_n\Set{ L_n / \log n,~ W_n / N_n,~ S_n / (N_n \log n),~ \Xi }.
    \end{equation} 
    Let $\tilde{N}_n = C_N N_n$,
    $$\mathcal{F}_n =  \UB(F) \cap \left( \bigcup_{N=1}^{\tilde{N}_n} \Phi(\tilde{L}_n({{H}}_0),~\tilde{W}_n({{H}}_0,N),~\tilde{S}_n({{H}}_0,N),~\tilde{B}_n({{H}}_0, N)) \right)$$
    for sufficiently large $C_N > 0$ and $\pi_N(N)$ be a density function of $N$.
    First, show that (a) holds. From \autoref{lem:covering},
    \begin{equation}\label{eqn:adaptive_space_covering}
        \begin{aligned}
            &N\left( \frac{\epsilon_n}{36}, \mathcal{F}_n, \norm{\cdot}_{L^\infty} \right)\\
            &\leq \sum_{N=1}^{\tilde{N}_n} \left( \frac{72}{\epsilon_n} \tilde{L}_n({{H}}_0) (\tilde{B}_n({{H}}_0, N) \vee 1 )^{\tilde{L}_n({{H}}_0)} (\tilde{W}_n({{H}}_0,N)+1)^{2\tilde{L}_n({{H}}_0)} \right)^{\tilde{S}_n({{H}},N)+1} \\
            &\leq \tilde{N}_n \left( \frac{72}{\epsilon_n} \tilde{L}_n({{H}}_0)(\tilde{B}_n({{H}}_0, \tilde{N}_n) \vee 1 )^{\tilde{L}_n({{H}}_0)} (\tilde{W}_n({{H}}_0, \tilde{N}_n)+1)^{2\tilde{L}_n({{H}}_0)} \right)^{\tilde{S}_n({{H}}_0, \tilde{N}_n)+1}
        \end{aligned}
    \end{equation}
    and 
    \begin{equation}
        \begin{aligned}
            &\sup_{\epsilon > \epsilon_n} \log N\left( \frac{\epsilon}{36}, A_{\epsilon, 1} \cap \mathcal{F}_n, \norm{\cdot}_n \right) \\
            &\leq \sup_{\epsilon > \epsilon_n} \log N\left( \frac{\epsilon}{36}, A_{\epsilon, 1} \cap \mathcal{F}_n, \norm{\cdot}_{L^\infty} \right) \\
            &\leq \sup_{\epsilon > \epsilon_n} \log N\left( \frac{\epsilon}{36}, \mathcal{F}_n, \norm{\cdot}_{L^\infty} \right) \\
            &\leq \log N\left( \frac{\epsilon_n}{36}, \mathcal{F}_n, \norm{\cdot}_{L^\infty} \right) \\
            &\leq \log \tilde{N}_n \\
            &\quad +\left[ {\tilde{S}_n({{H}}_0, \tilde{N}_n) + 1} \right]\log \left(  \frac{72}{\epsilon_n} \tilde{L}_n({{H}}_0) (\tilde{B}_n({{H}}_0, \tilde{N}_n) \vee 1)^{\tilde{L}_n({{H}}_0)} (\tilde{W}_n({{H}}_0,\tilde{N}_n) + 1)^{2\tilde{L}_n({{H}}_0)}\right) \\
            &\lesssim \tilde{N}_n (\log n)^3\\
            &\lesssim n\epsilon_n^2.
        \end{aligned}
    \end{equation}
    for sufficiently large $n$. Next, show that (b) holds. Note $N_n (\log n)^3 \lesssim n \epsilon_n^2$ and
    \begin{equation}
        L_n \leq \tilde{L}_n({{H}}_n),~W_n \leq \tilde{W}_n({{H}}_n,N_n),~S_n \leq \tilde{S}_n({{H}}_n,N_n),~B_n \leq \tilde{B}_n({{H}}_n , \tilde{N}_n)
    \end{equation}
    for $N_n,~L_n,~W_n,~S_n,~B_n$ in \autoref{thm:main} and sufficiently large $n$. Thus, there exists a constant $D > 0$ such that 
    \begin{equation}
        \begin{gathered}
            \pi_N(N_n) \gtrsim \Exp{-N_n (\log n)^{2} \log \frac{N_n}{\lambda_N}} \gtrsim \Exp{-Dn \epsilon_n^2}
        \end{gathered} 
    \end{equation}
    and
    \begin{equation}
        \begin{aligned}
        &\Pi(f_\theta \in \mathcal{F}_n: \norm{f - f_0}_n \leq \epsilon_n) \\
        &\geq \pi_N(N_n) \Pi(f_\theta \in \Phi(L_n, W_n, S_n, B_n): \norm{f_\theta-f_0}_n \leq \epsilon_n |N_n) \\
        &\gtrsim \Exp{-(C+D)n\epsilon_n^2}
        \end{aligned}
    \end{equation}
    holds for sufficiently large $n$. (b) holds for $C^\prime = \max\Set{C+D, 1 + 2 / \sigma^2}$.
    From $$\Pi(\mathcal{F} \setdiff \mathcal{F}_n) \leq \pi_N(N>\tilde{N}_n)$$
    and \textit{Chernoff bound}, for any positive number $t,~{{Z}}_0>0$,
    \begin{equation}
        \begin{gathered}
            P({{Z}}>{{Z}}_0) < e^{-t({{Z}}_0+1)} \expt{e^{t{{Z}}}} \lesssim e^{-t({{Z}}_0+1)} \left(\Exp{e^t \lambda_N} -1\right).
        \end{gathered}
    \end{equation}
    Letting $t=\log {{Z}}_0$, we get
    \begin{equation}
        P({{Z}}> {{Z}}_0) \lesssim e^{-({{Z}}_0+1) \log {{Z}}_0} \left(\Exp{{{Z}}_0 \lambda_N} -1\right).
    \end{equation}
    Thus,
    \begin{equation}\label{eqn:thm2:c}
        \begin{gathered}
            \pi_N(N>\tilde{N}_n) \lesssim e^{-\left[(\tilde{N}_n+1)\log \tilde{N}_n + \tilde{N}_n \lambda_N \right] (\log n)^{2}}, \\
            (C^\prime \sigma^2 +2)n\epsilon_n^2 + \lambda_N \tilde{N}_n (\log n)^{2} - (\tilde{N}_n+1)\log \tilde{N}_n (\log n)^{2} \rightarrow - \infty
        \end{gathered}
    \end{equation}
    for sufficiently large $C_N>0$. (c) holds.
\end{proof}

\subsection{Proof of Theorem \ref{thm:shrinkage}}\label{subsec:proof_shrinkage}

\begin{proof}
    Let $\mathcal{F} = \Phi \cap \UB(F)$. From \autoref{lem:consistency}, it is enough to show that there exist a constant $C^\dprime > 2 / \sigma^2$ and $\mathcal{F}_n \subset \mathcal{F}$ which satisfy
    \begin{enumerate}[(a)]
        \item $\sup_{\epsilon > \epsilon_n} \log N(\epsilon/36, A_{\epsilon, 1} \cap \mathcal{F}_n, \norm{\cdot}_n) \leq n\epsilon_n^2$
        \item $-\log \Pi(A_{\epsilon_n, 1}) \leq C^\dprime n \epsilon_n^2$
        \item $\Pi(\mathcal{F} \setdiff \mathcal{F}_n) = o\left(e^{-(C^\dprime \sigma^2 +2)n \epsilon_n^2}\right)$
    \end{enumerate}
    for sufficiently large $n$. Let $\mathcal{F}_n = \Phi(L_n, W_n, S_n, B_n, a_n) \cap \UB(F)$. It is easy to show that (a) holds from \autoref{lem:covering_a} as in the proof of \autoref{thm:main}. By the assumption,
    \begin{equation}
        \begin{aligned}
            &\Pi(\mathcal{F} \setdiff \mathcal{F}_n) \\
            &\leq \pi\left( ^\exists \abs{\theta_i} > B_n\right | L_n,W_n,S_n,B_n) + \pi\left(\sum_{i=1}^{T_n} I(\abs{\theta_i} >a_n) > S_n | L_n,W_n,S_n,B_n \right) \\
            &= \left(1 - (1-v_n)^{T_n}\right) + P \left( S > S_n | S \sim B(T_n, 1 - u_n) \right) \\
            &\leq T_n v_n + \exp\left( - T_n \left\{ \left(1-\frac{S_n}{T_n} \right) \log \frac{1-S_n/T_n}{u_n} + \frac{S_n}{T_n} \log\frac{S_n/T_n}{1-u_n}  \right\} \right)\\
            &= o\left(e^{-K_0n\epsilon_n^2 + \log T_n}\right) + \Exp{ T_n \left(1-\frac{S_n}{T_n} \right) \log \frac{u_n}{1-S_n/T_n} - S_n \log\frac{S_n/T_n}{1-u_n} } \\
            &\leq o\left(e^{-K_1n\epsilon_n^2}\right) + \Exp{ T_n \left(1-\frac{S_n}{T_n} \right) \log \frac{1 - \eta_n S_n/T_n}{1-S_n/T_n} - S_n \log\frac{S_n/T_n}{1-u_n} } \\
            &\leq o\left(e^{-K_1n\epsilon_n^2}\right) + \Exp{ T_n \left(1-\frac{S_n}{T_n} \right) \left(\frac{(1-\eta_n) S_n/T_n}{1-S_n/T_n} + o\left(\frac{(1-\eta_n) S_n/T_n}{1-S_n/T_n}\right)\right) -K n\epsilon_n^2 } \\
            &= o\left(e^{-K_1n\epsilon_n^2}\right) + \Exp{S_n(1-\eta_n) + o(S_n(1-\eta_n)) -K n\epsilon_n^2 } \\
            &= o\left(e^{-K_1n\epsilon_n^2}\right) + o\left(e^{-K_2 n\epsilon_n^2}\right) \\
            &= o\left(e^{-\min\{K_1, K_2\}n\epsilon_n^2}\right).
        \end{aligned}
    \end{equation}
    for some $4<K_1<K$ and $4 < K_2 < K_0$. Letting $C^\dprime = (\min\{K_1, K_2\}-2)/\sigma^2$, (c) holds. We use Bernoulli's inequality and a tail bound for binomial distribution in \citet{arratia1989tutorial} for the second inequality. Next, as in the proof of \autoref{thm:main}, there is a constant $C_1>0$ and $\hat{f}_n = f_{\hat{\theta}} \in \mathcal{F}_n$ such that
    \begin{equation}
        \begin{gathered}
        \norm{\hat{f}_n - f_0}_{L^2} \leq C_1 \norm{f_0}_{B^s_{p,q}([0, 1]^d)} N_n^{-s/d} \leq \epsilon_n/4, \\
        \norm{f - f_0}_n^2 \leq 4 \norm{f - f_0}_{L^2}^2
        \end{gathered}
    \end{equation}
    for sufficiently large $n$ almost surely. Let $\hat{\gamma}$ and $\hat{\theta}_{\hat{\gamma}}$ be index and value of nonzero components of $\hat{\theta}$ respectively. Let 
    $$\tilde{\Theta}(L, W, S, B, a) = \left\{\tilde{\theta}: \theta \in \Theta(L, W, S, B, a)\right\} = \Theta(L, W, S, B)$$
    and $\tilde{\Theta}(\hat{\gamma}; L_n, W_n, S_n, B_n, a_n)$ be a subset of parameter space $\tilde{\Theta}(L_n, W_n, S_n, B_n, a_n)$ consists of parameters which have nonzero components at $\hat{\gamma}$ only and $\mathcal{F}_n(\hat{\gamma}) = \tilde{\Phi}(\hat{\gamma}; L_n, W_n, S_n, B_n, a_n) \cap \UB(F)$ be an uniformly bounded neural network space generated by $\tilde{\Theta}(\hat{\gamma}; L_n, W_n, S_n, B_n, a_n)$. 
    Note 
    % \begin{equation}
    %     \norm{f_\theta - f_{\tilde{\theta}}}_{L^\infty} \leq a_n \leq \epsilon_n / 4 
    % \end{equation}
    % for any $\theta \in \Theta(L_n, W_n, T_n) = \bigcup_{B=0}^\infty \Theta(L_n, W_n, T_n, B)$ and
    \begin{equation}
        \begin{aligned}
            \Pi(A_{\epsilon_n, 1} )   
            &= \Pi(f \in \mathcal{F}: \norm{f - f_0}_n \leq \epsilon_n) \\
            &\geq \Pi\left(f \in \mathcal{F}: \norm{f - \hat{f}_n}_{L^2} \leq \epsilon_n/4\right) \\
            &\geq \Pi\left(f \in \mathcal{F}: \norm{f - \hat{f}_n}_{L^\infty} \leq \epsilon_n/4\right) \\
            % &\geq \Pi\left(f \in \mathcal{F}_n: \norm{f - \hat{f}_n}_{L^\infty} \leq \epsilon_n/2\right) \\
            &\geq \Pi\left(f \in \mathcal{F}_n(\hat{\gamma}): \norm{f - \hat{f}_n}_{L^\infty} \leq \epsilon_n/4\right)
        \end{aligned}
    \end{equation}
    for sufficiently large $n$. As in the proof of \autoref{lem:covering},
    % \begin{equation}
    %     \begin{aligned}
    %          &\Pi\bigg( f \in \mathcal{F}_n(\hat{\gamma}): \norm{f - \hat{f}_n}_{L^\infty} \leq \epsilon_n/2\bigg) \\
    %          &\geq \Pi\left(\theta \in \mathbb{R}^{T_n}: \theta_{\hat{\gamma}^c} \in [-a_n, a_n]^{T_n-S_n},~ \norm{\theta_{\hat{\gamma}}}_{\infty} \leq B_n,~ \norm{\hat{\theta}_{\hat{\gamma}} - \theta_{\hat{\gamma}}}_{\infty} \leq \frac{\epsilon_n}{2 (W_n+1)^{L_n} L_n (B_n \vee 1)^{L_n-1}} \right) \\
    %          &\geq \left(1 - \Exp{-S_n \left(\frac{\epsilon_n}{2 (W_n+1)^{L_n} L_n (B_n \vee 1)^{L_n-1}}\right)^{2/\xi} } \right)^{S_n} u_n^{T_n - S_n} 
    %     \end{aligned}
    % \end{equation}
    % Letting $$y_n = \Exp{-S_n \left(\frac{\epsilon_n}{2 (W_n+1)^{L_n} L_n (B_n \vee 1)^{L_n-1}}\right)^{2/\xi} },$$ 
    % \begin{equation}
    %     \begin{aligned}
    %          -\log \Pi(A_{\epsilon_n, 1} ) 
    %          &\leq - \log \Pi\bigg( f \in \mathcal{F}_n(\hat{\gamma}): \norm{f - \hat{f}_n}_{L^\infty} \leq \epsilon_n/4\bigg)\\
    %          &= -S_n \Log{1 - y_n} - (T_n - S_n) \log u_n \\
    %          &\leq -S_n \Log{1 - y_n} - T_n\left(1 - \frac{S_n}{T_n}\right) \Log{1-\eta_nS_n/T_n} \\
    %          &= -S_n \Log{1 - y_n} + T_n\left(1 - \frac{S_n}{T_n}\right) \left( \eta_nS_n/T_n + o\left(\eta_nS_n /T_n \right) \right) \\
    %          &= S_n (y_n + o(y_n)) + S_n \eta_n + o\left(S_n\eta_n\right) \\
    %          &\lesssim C n\epsilon_n^2.
    %     \end{aligned}
    % \end{equation}
    \begin{equation}
        \begin{aligned}
             &\Pi\bigg( f \in \mathcal{F}_n(\hat{\gamma}): \norm{f - \hat{f}_n}_{L^\infty} \leq \epsilon_n/4\bigg) \\
             &\geq \Pi\bigg(\theta \in \mathbb{R}^{T_n}: \theta_{\hat{\gamma}^c} \in [-a_n, a_n]^{T_n-S_n},~ \norm{\theta_{\hat{\gamma}}}_{\infty} \leq B_n, \\
                &\qquad \norm{\hat{\theta}_{\hat{\gamma}} - \theta_{\hat{\gamma}}}_{\infty} \leq \frac{\epsilon_n}{4 (W_n+1)^{L_n} L_n (B_n \vee 1)^{L_n-1}} \bigg) \\
             &\geq u_n^{T_n - S_n} \left( \int_{B_n - t_n}^{B_n} g(t) dt \right)^{S_n},
        \end{aligned}
    \end{equation}
    where $t_n = \frac{\epsilon_n}{4 (W_n+1)^{L_n} L_n (B_n \vee 1)^{L_n-1}}$.
    Letting $$y_n =\int_{B_n - t_n}^{B_n} g(t) dt \geq t_n g(B_n),$$ 
    \begin{equation}
        \begin{aligned}
             -\log \Pi(A_{\epsilon_n, 1} ) 
             &\leq -S_n \log y_n - (T_n - S_n) \log u_n \\
             &\leq -S_n \Log{t_n g(B_n) } - T_n\left(1 - \frac{S_n}{T_n}\right) \Log{1-S_n/T_n} \\
             &= -S_n \log t_n - S_n \log g(B_n) + T_n\left(1 - \frac{S_n}{T_n}\right) \left( S_n/T_n + o\left(S_n /T_n \right) \right) \\
             &\lesssim S_n (\log n)^2 + S_n + o\left(S_n\right) \\
             &\lesssim n\epsilon_n^2.
        \end{aligned}
    \end{equation}
\end{proof}

\section{Numerical Examples}\label{appendix:numerical_results}

\subsection{Gaussian Mixture Prior}\label{subsec:Ex-Gaussian-Mixture}

Consider a problem of estimating the Besov functions $f_1,~f_2,~f_3$ and $f_4$. Assume the Gaussian mixture prior distribution for each parameter as in Example \ref{ex:gaussian_mixture}. We sampled $n$ points from $x \sim U(0, 1)$ and $y$ from $\Gaussian{f_i(x)}{\sigma_i^2}$ for $i=1,2,3$ and $4$. We set $\sigma_1^2= \sigma_2^2=0.01^2$ and $\sigma_3^2=\sigma_4^2=0.1^2$.

As mentioned in \autoref{sec:example}, we used a smaller model that was weaker than the conditions of the theorem. We fix the depth $L_n = 5$, the width $W_n=200$ and set $N_n \leftarrow W_n / W_0$ of the model. We consider the prior in \autoref{ex:gaussian_mixture} (Gaussian mixture BNN) and set the smaller variance $\sigma_{1n} \leftarrow \max\{ 0.001,~\sigma_{1n} \}$ to avoid numerical precision issues.

% We summarize the hyperparameters of the model used in the experiment in \autoref{table:f1_networks_mini} and \autoref{table:f2_networks_mini}.
% \begin{table}[ht]
%     \centering
%     \caption{Hyperparameters of Gaussian mixture BNN to estimate $f_1 \in B_{\infty, \infty}^{\log2 / \log 3}([0, 1])$.}
%     \label{table:f1_networks_mini}
%     \begin{tabular}{lrrrrrr}
%     \toprule
%      & $L$ & $W$ & $\sigma_{1n}$ & $\sigma_{2n}$ & $\pi_{1n}$ & $\pi_{2n}$ \\
%     \midrule
%     100 & 5 & 200 & $1.0\times10^{-3}$ & $4.221\times10^{-1}$ & $6.699\times10^{-1}$ & $3.300\times10^{-1}$ \\
%     1,000 & 5 & 200 & $1.0\times10^{-3}$ & $1.381\times10^{-1}$ & $6.699\times10^{-1}$ & $3.300\times10^{-1}$ \\
%     10,000 & 5 & 200 & $1.0\times10^{-3}$ & $5.392\times10^{-2}$ & $6.699\times10^{-1}$ & $3.300\times10^{-1}$ \\
%     \bottomrule
%     \end{tabular}
% \end{table}

% \begin{table}[ht]
%     \centering
%     \caption{Hyperparameters of Gaussian mixture BNN to estimate $f_2,~f_3$ and $f_4 \in B^{1}_{1, 1}([0, 1])$.}
%     \label{table:f2_networks_mini}
%     \begin{tabular}{lrrrrrr}
%     \toprule
%      & $L$ & $W$ & $\sigma_{1n}$ & $\sigma_{2n}$ & $\pi_{1n}$ & $\pi_{2n}$ \\
%     \midrule
%     100 & 5 & 200 & $1.0\times10^{-3}$ & $5.423\times10^{-1}$ & $6.699\times10^{-1}$ & $3.300\times10^{-1}$ \\
%     1,000 & 5 & 200 & $1.0\times10^{-3}$ & $2.011\times10^{-1}$ & $6.699\times10^{-1}$ & $3.300\times10^{-1}$ \\
%     10,000 & 5 & 200 & $1.0\times10^{-3}$ & $8.899\times10^{-2}$ & $6.699\times10^{-1}$ & $3.300\times10^{-1}$ \\
%     \bottomrule
%     \end{tabular}
% \end{table}    

We fit the model using the NUTS algorithm \citep{hoffman2014no} with the python \texttt{Pyro} \citep{bingham2019pyro} and \texttt{PyTorch} \citep{NEURIPS2019_9015} packages. Experiments were run on a GPU server with Nvidia GeForce GTX TITAN X and RTX 3090. The code and instructions for the experiment are provided in the supplementary material. \autoref{fig:function_2GMM_results} shows the results. Overall, as the number of the data $n$ increase, the mean functions get closer to the true regression function.

\begin{figure}[!ht]
    \centering
    \caption{The results of estimating four functions $f_1,~f_2,~f_3$ and $f_4$ using Gaussian mixture BNN model with $n=100$ (left), $n=1,000$ (center) and $n=10,000$ (right) samples. We construct 1,000 functions from the MCMC samples and plot the mean function with training data. The blue lines are the mean functions and the yellow intervals are the prediction intervals.}\label{fig:function_2GMM_results}
    \begin{subfigure}[b]{0.95\textwidth}
        \includegraphics[width=0.32\textwidth]{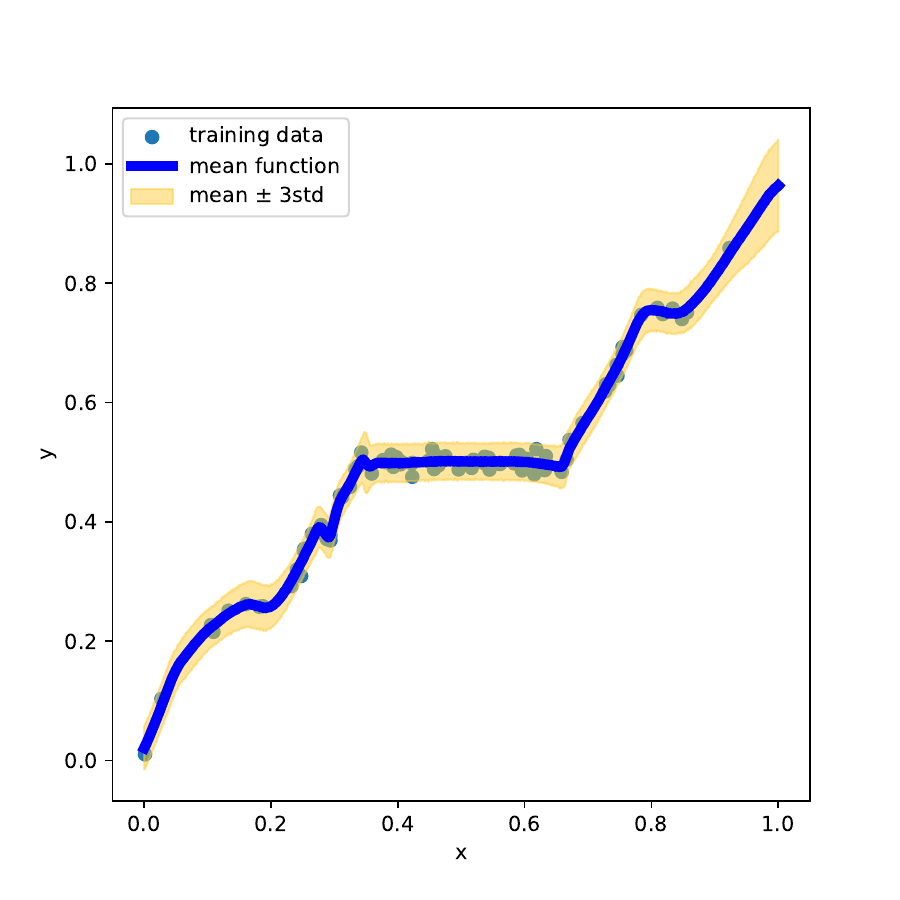}
        \includegraphics[width=0.32\textwidth]{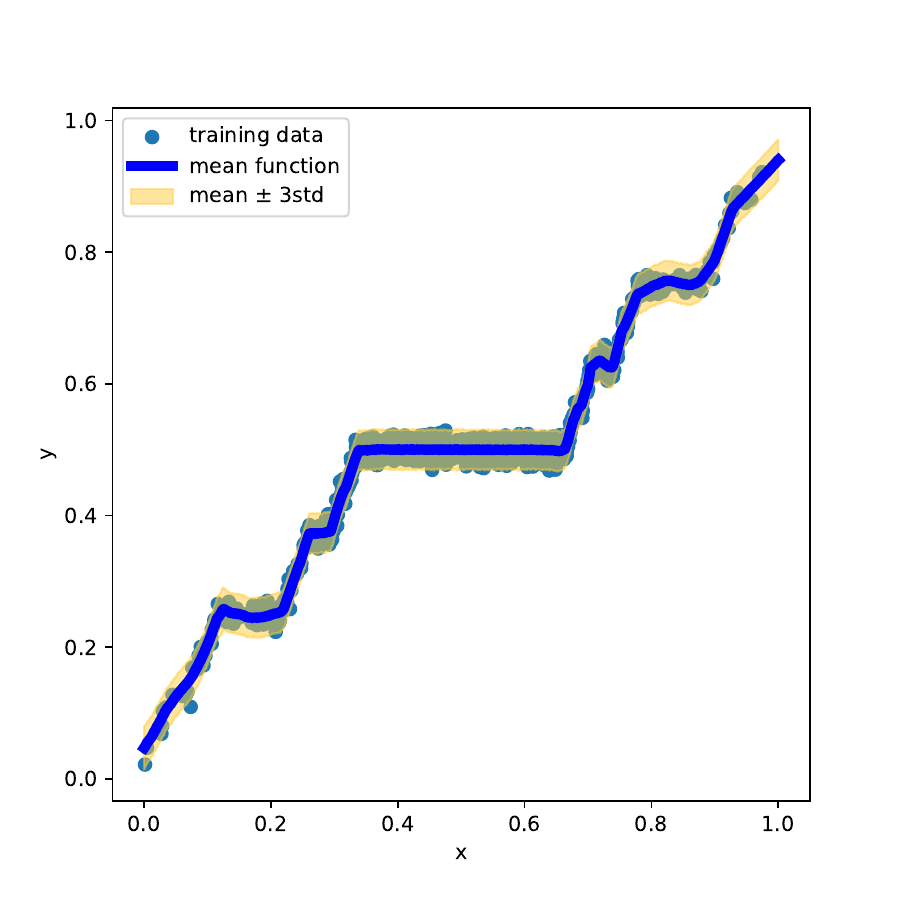}
        \includegraphics[width=0.32\textwidth]{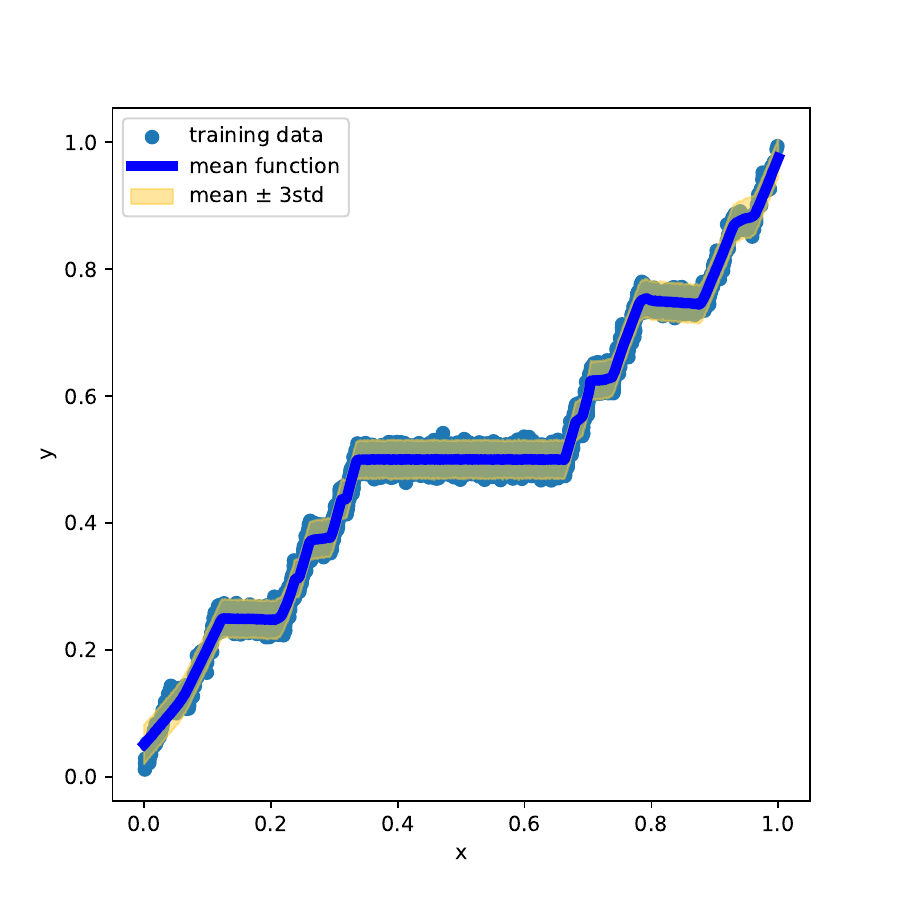}
        \caption{The results of estimating $f_1$.}
    \end{subfigure}
    \begin{subfigure}[b]{0.95\textwidth}
        \includegraphics[width=0.32\textwidth]{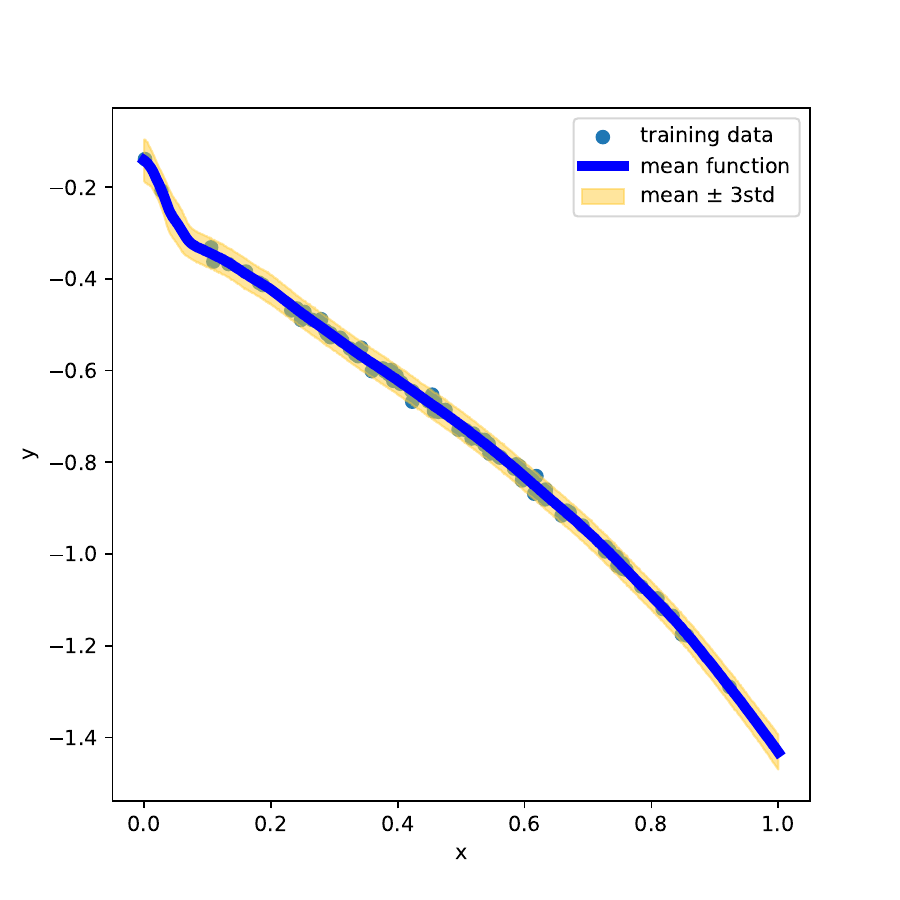}
        \includegraphics[width=0.32\textwidth]{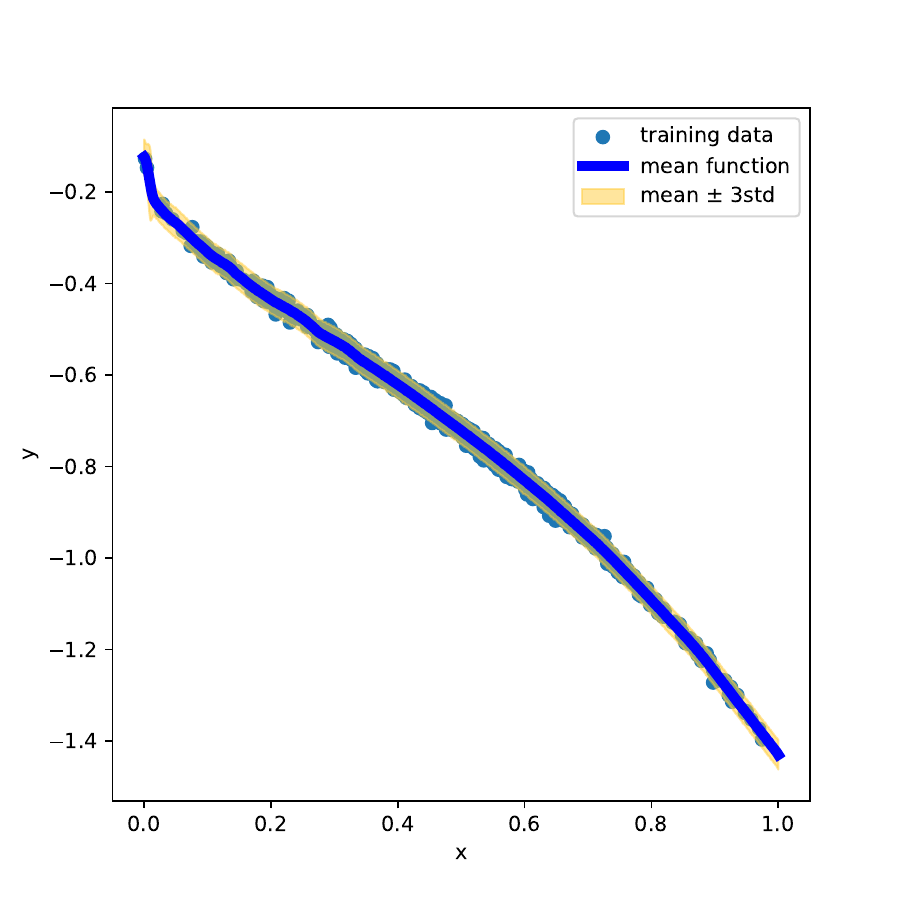}
        \includegraphics[width=0.32\textwidth]{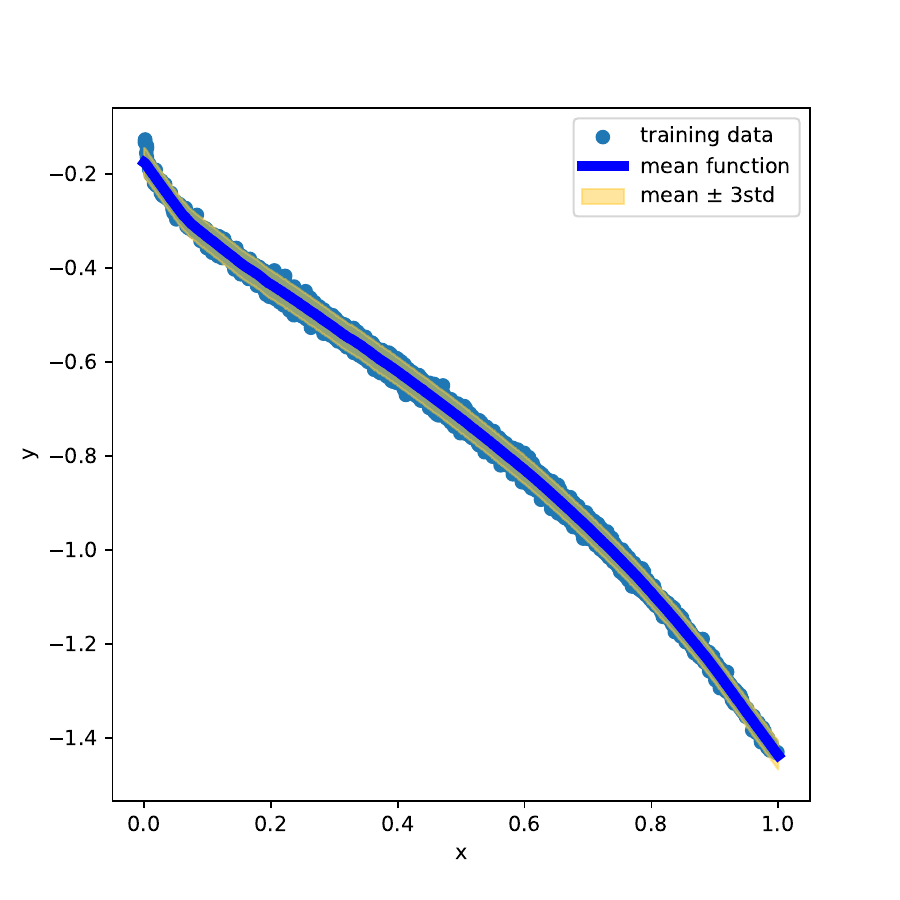}
        \caption{The results of estimating $f_2$.}
    \end{subfigure}
    \begin{subfigure}[b]{0.95\textwidth}
        \includegraphics[width=0.32\textwidth]{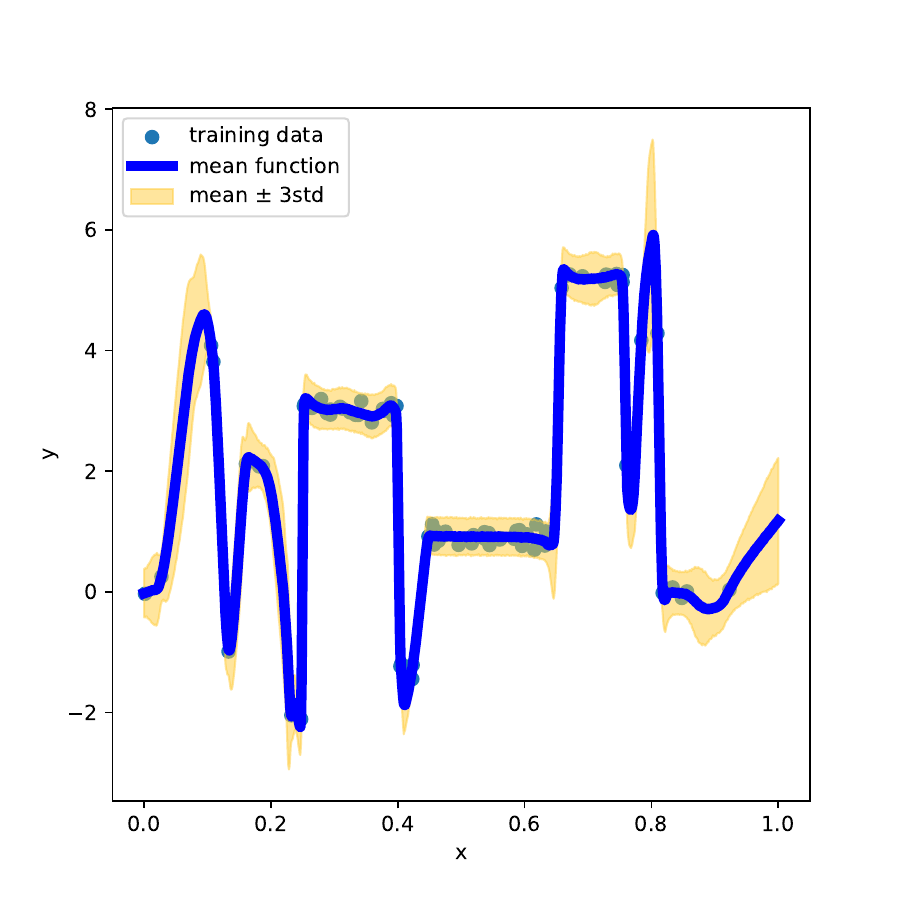}
        \includegraphics[width=0.32\textwidth]{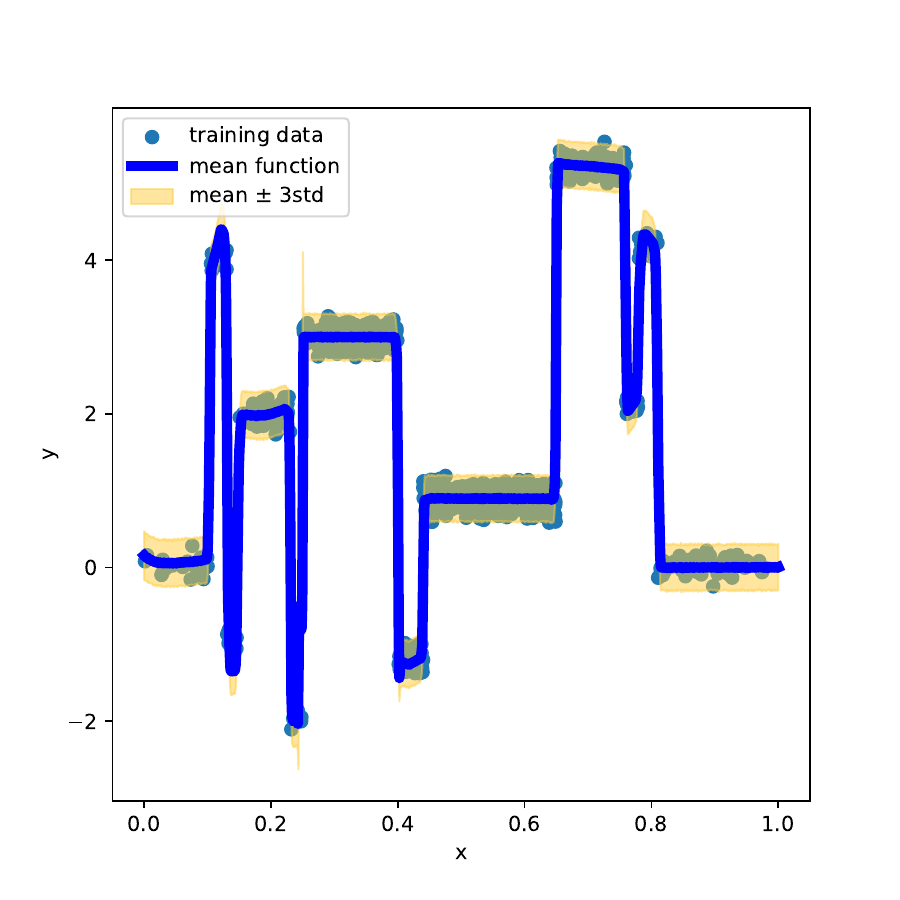}
        \includegraphics[width=0.32\textwidth]{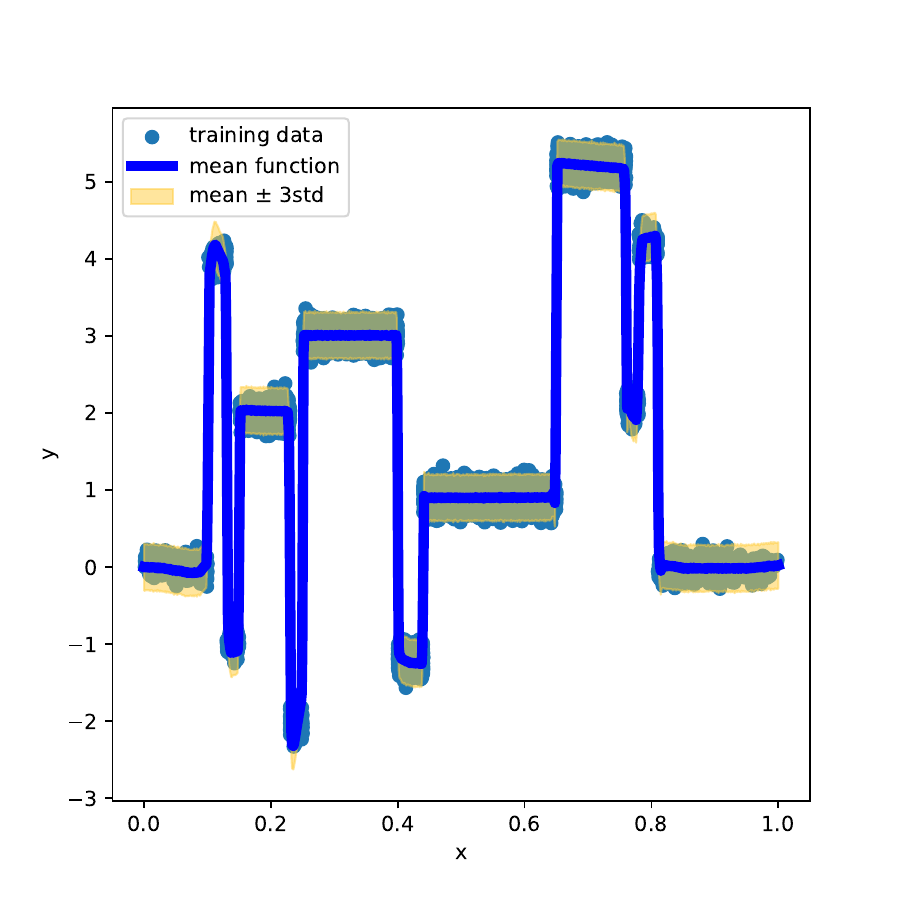}
        \caption{The results of estimating $f_3$.}
    \end{subfigure}
    \begin{subfigure}[b]{0.95\textwidth}
        \includegraphics[width=0.32\textwidth]{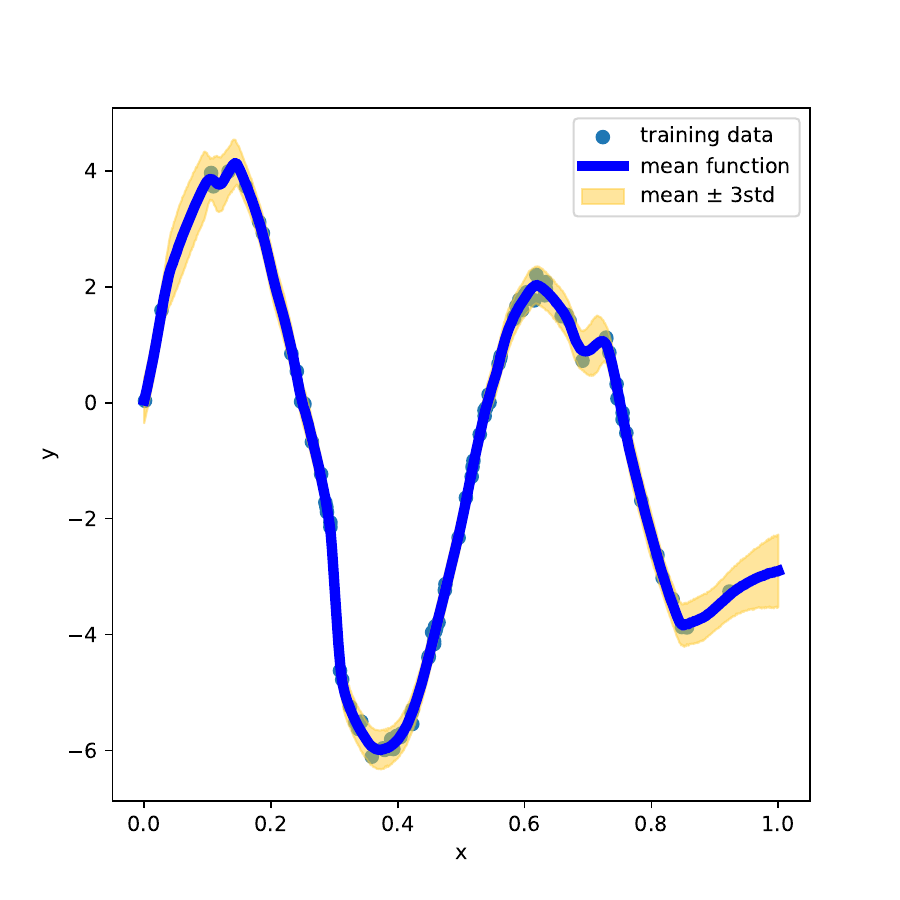}
        \includegraphics[width=0.32\textwidth]{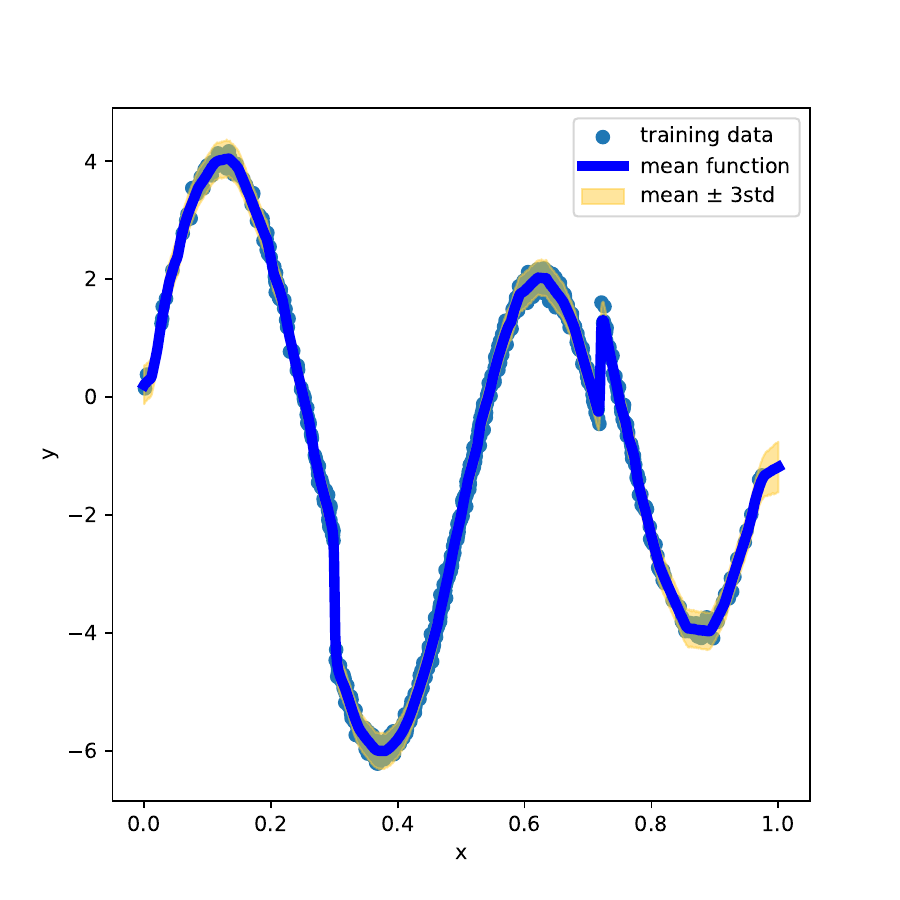}
        \includegraphics[width=0.32\textwidth]{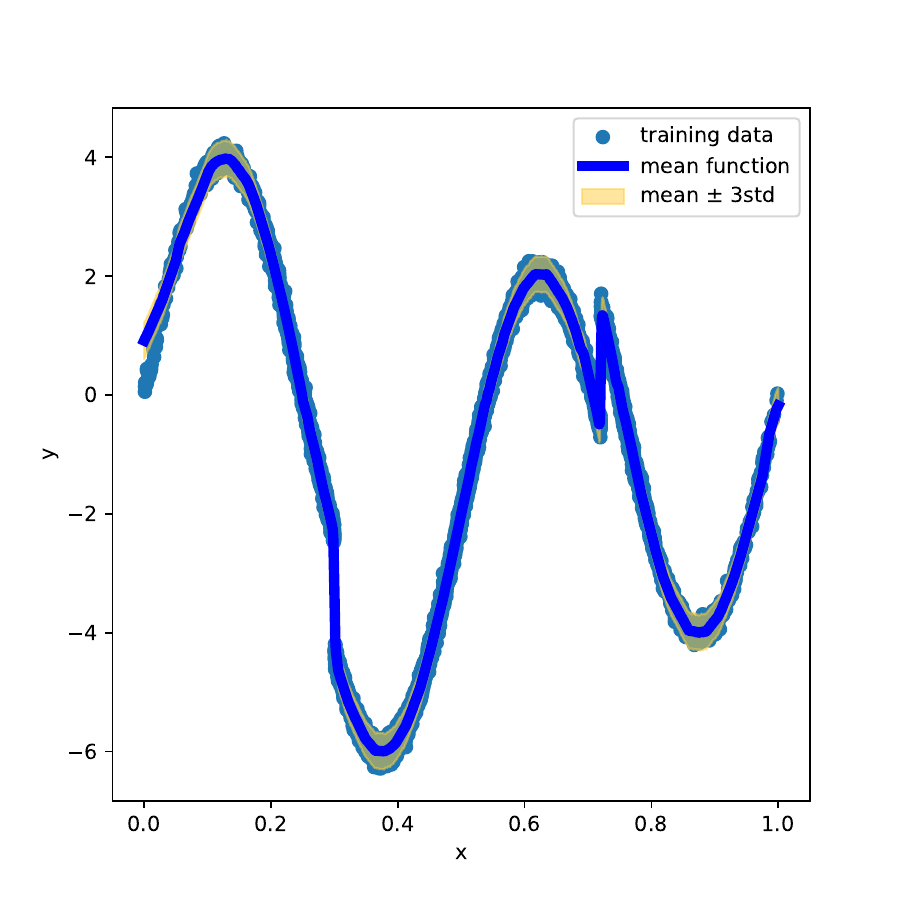}
        \caption{The results of estimating $f_4$.}
    \end{subfigure}
\end{figure}

\subsection{Gaussian Prior}

We performed the same experiment as before using Gaussian prior (Gaussian BNN). The standard deviation of the Gaussian prior is the same as $\sigma_{2n}$ in the Gaussian mixture BNN (see \autoref{ex:gaussian}). \autoref{fig:function_Gaussian_results} shows the results. Overall, as the number of the data $n$ increase, the mean functions get closer to the true regression function.

\begin{figure}[!ht]
    \centering
    \caption{The results of estimating four functions $f_1,~f_2,~f_3$ and $f_4$ using Gaussian BNN models with $n=100$ (left), $n=1,000$ (center) and $n=10,000$ (right) samples. We construct 1,000 functions from the MCMC samples and plot the mean function with training data. The blue lines are the mean functions and the yellow intervals are the prediction intervals.}\label{fig:function_Gaussian_results}
    \begin{subfigure}[b]{0.95\textwidth}
        \includegraphics[width=0.32\textwidth]{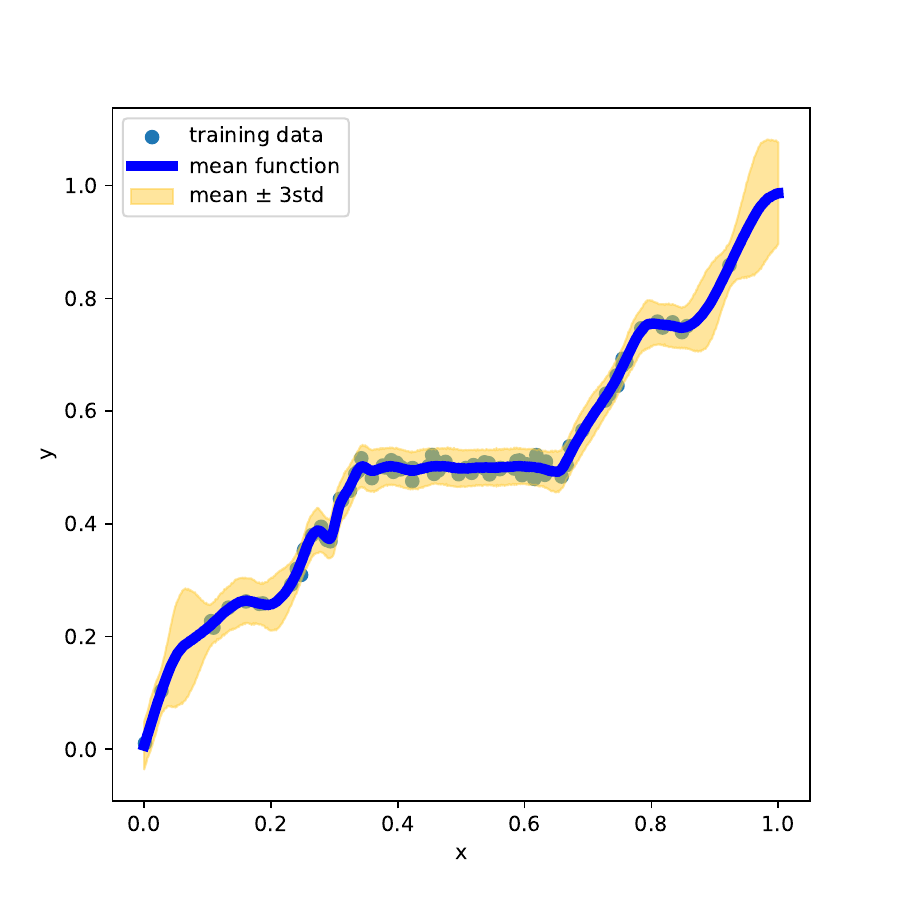}
        \includegraphics[width=0.32\textwidth]{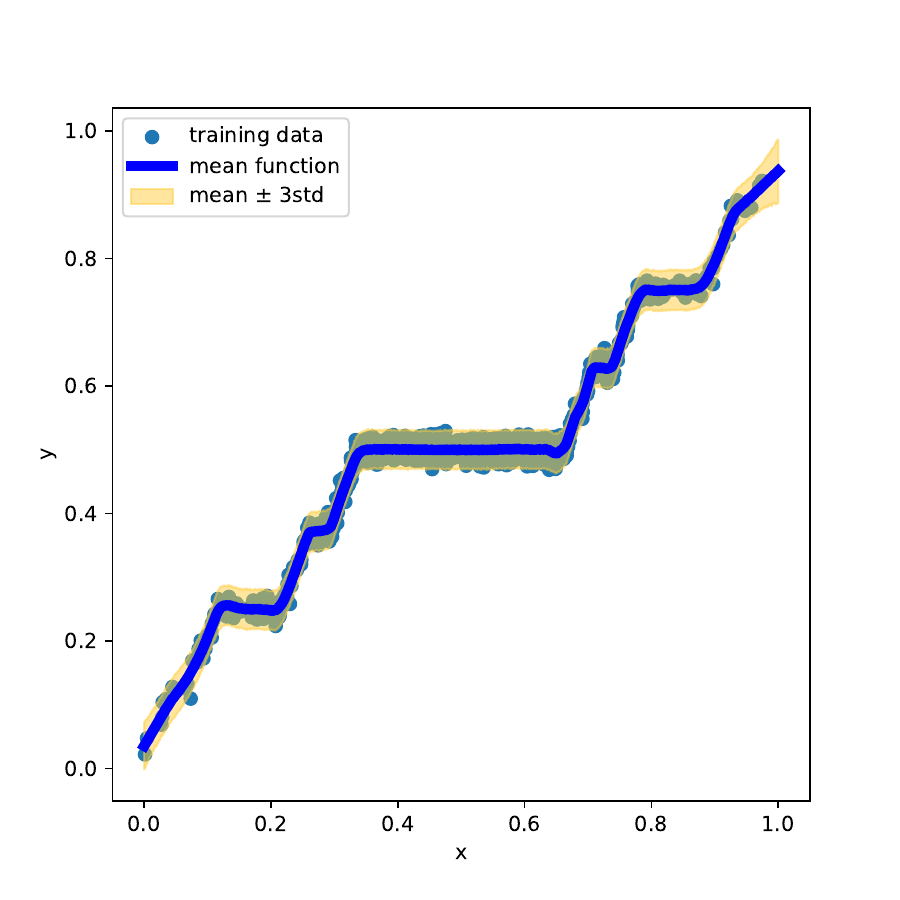}
        \includegraphics[width=0.32\textwidth]{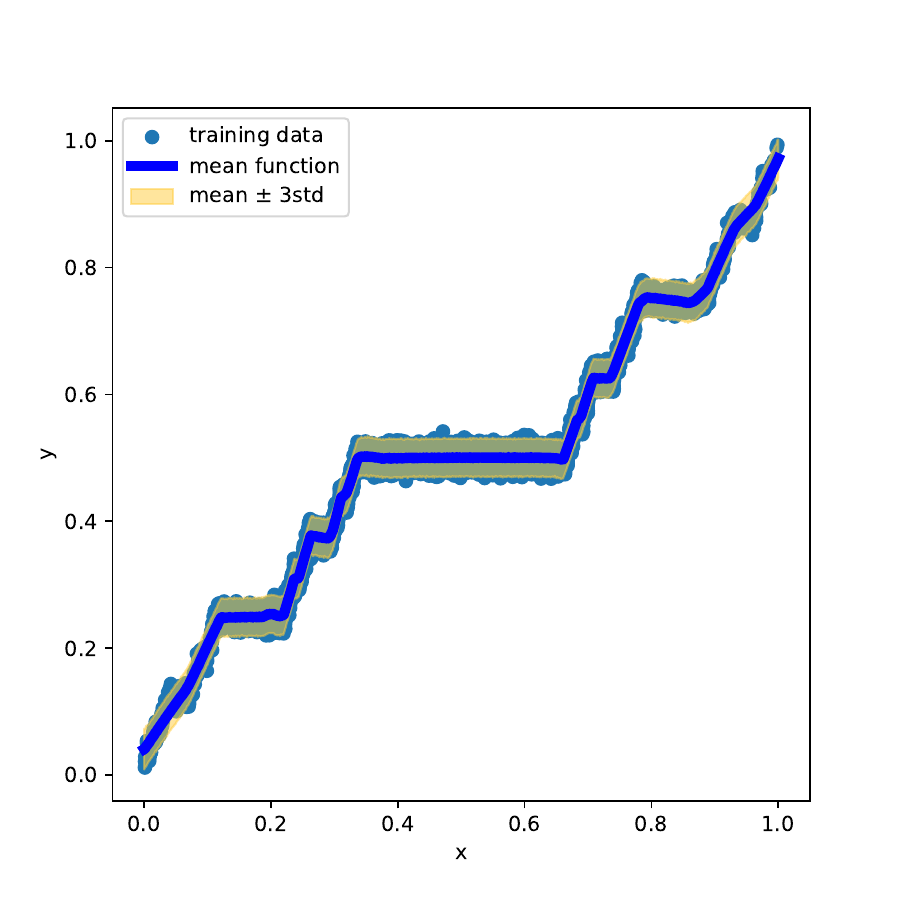}
        \caption{The results of estimating $f_1$.}
    \end{subfigure}
    \begin{subfigure}[b]{0.95\textwidth}
        \includegraphics[width=0.32\textwidth]{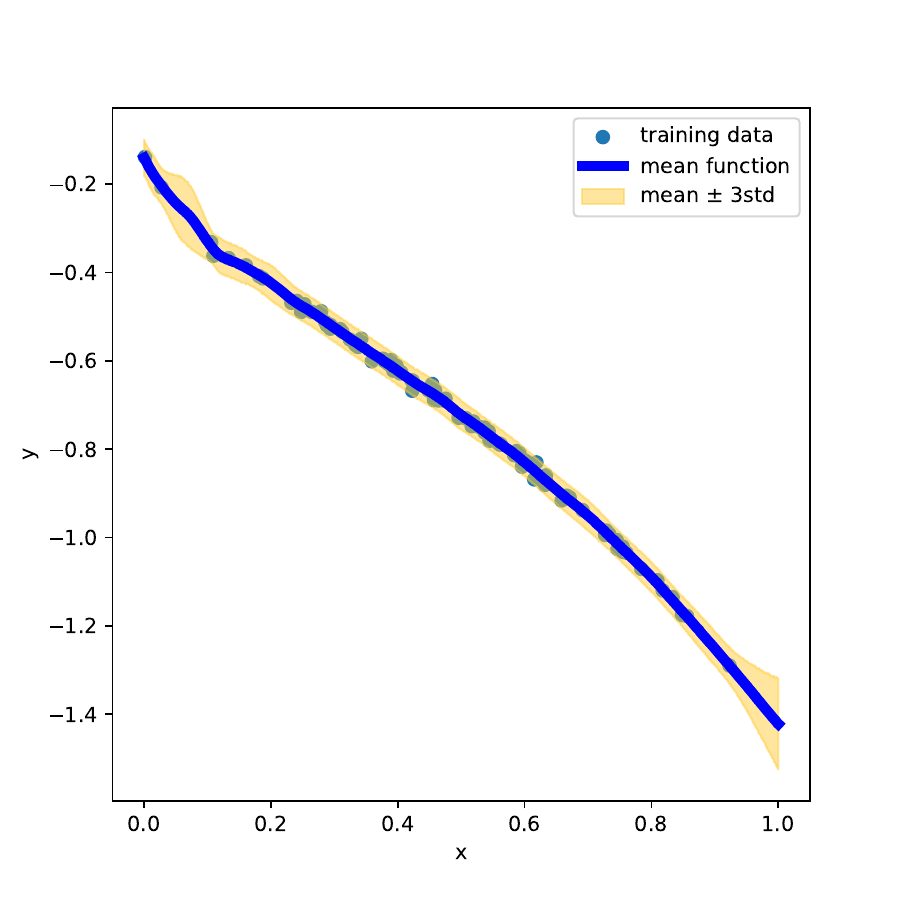}
        \includegraphics[width=0.32\textwidth]{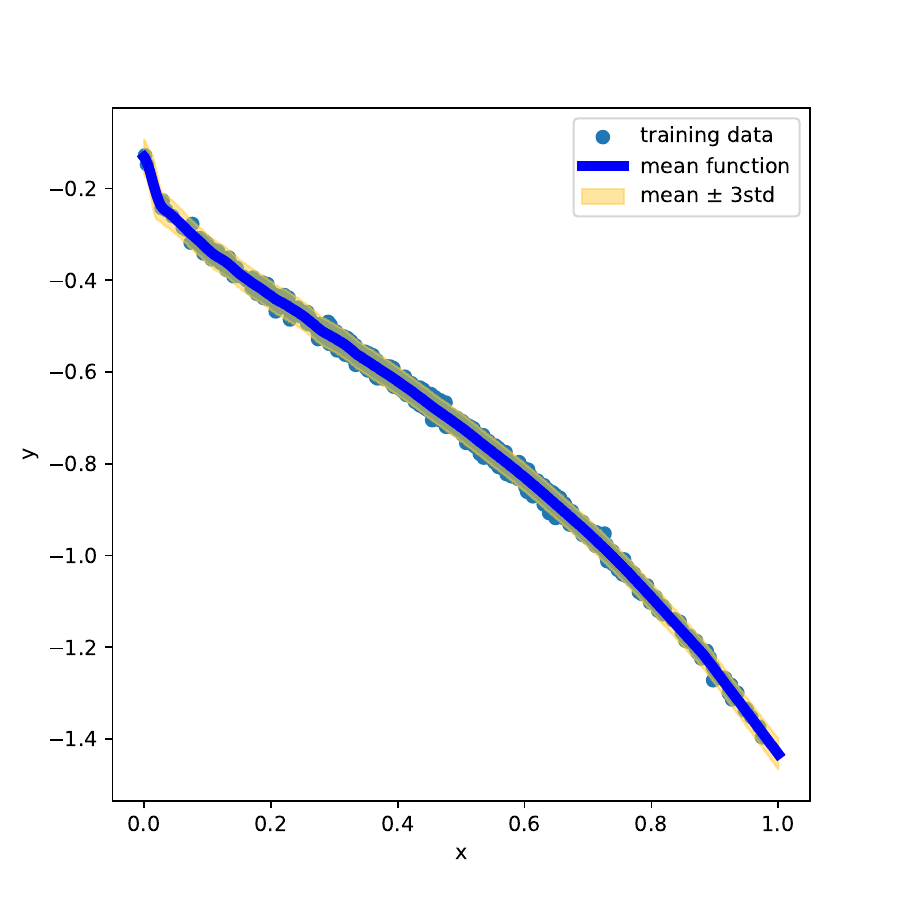}
        \includegraphics[width=0.32\textwidth]{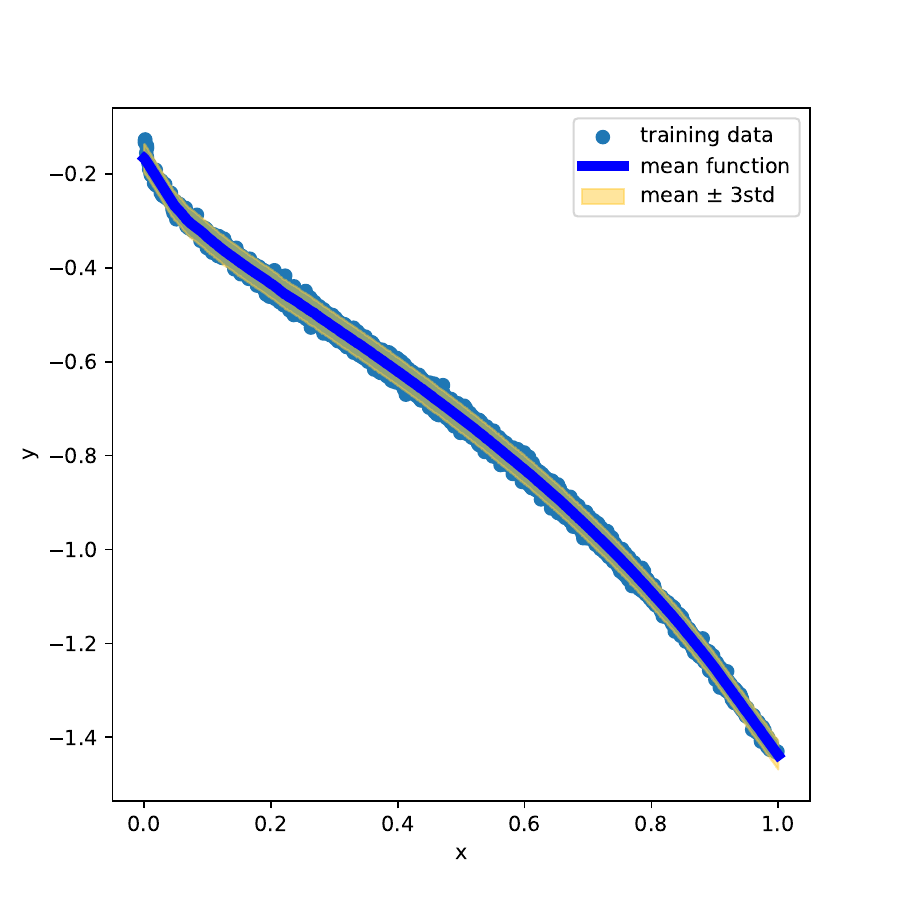}
        \caption{The results of estimating $f_2$.}
    \end{subfigure}
    \begin{subfigure}[b]{0.95\textwidth}
        \includegraphics[width=0.32\textwidth]{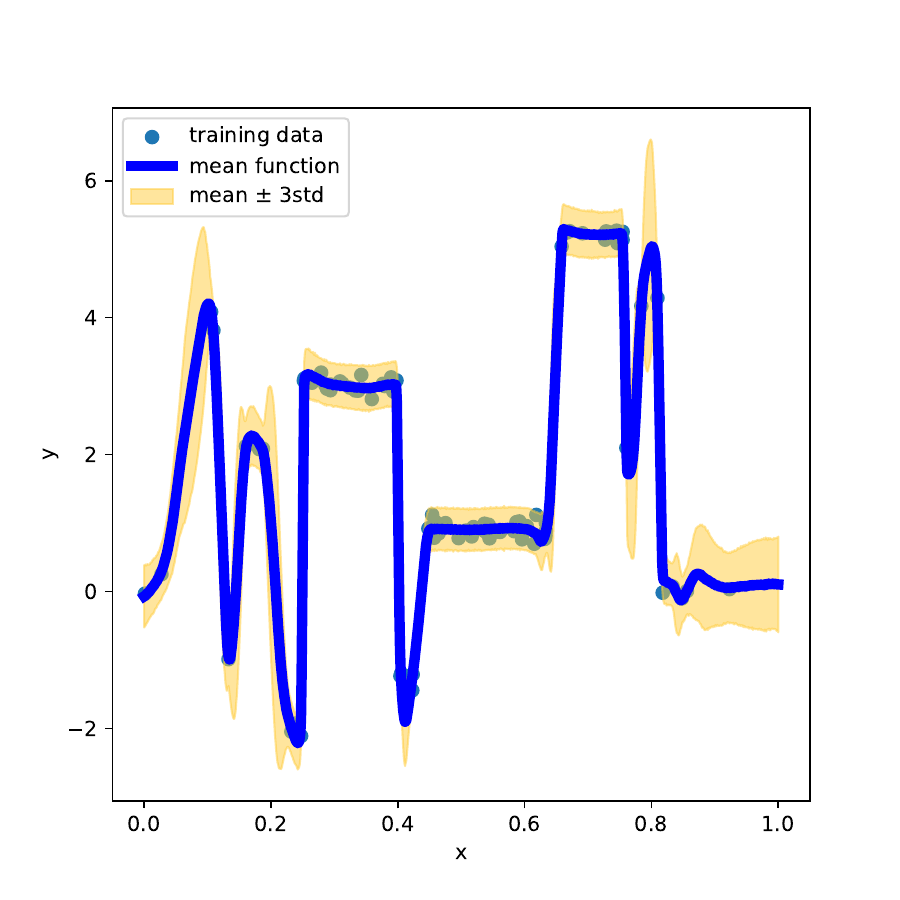}
        \includegraphics[width=0.32\textwidth]{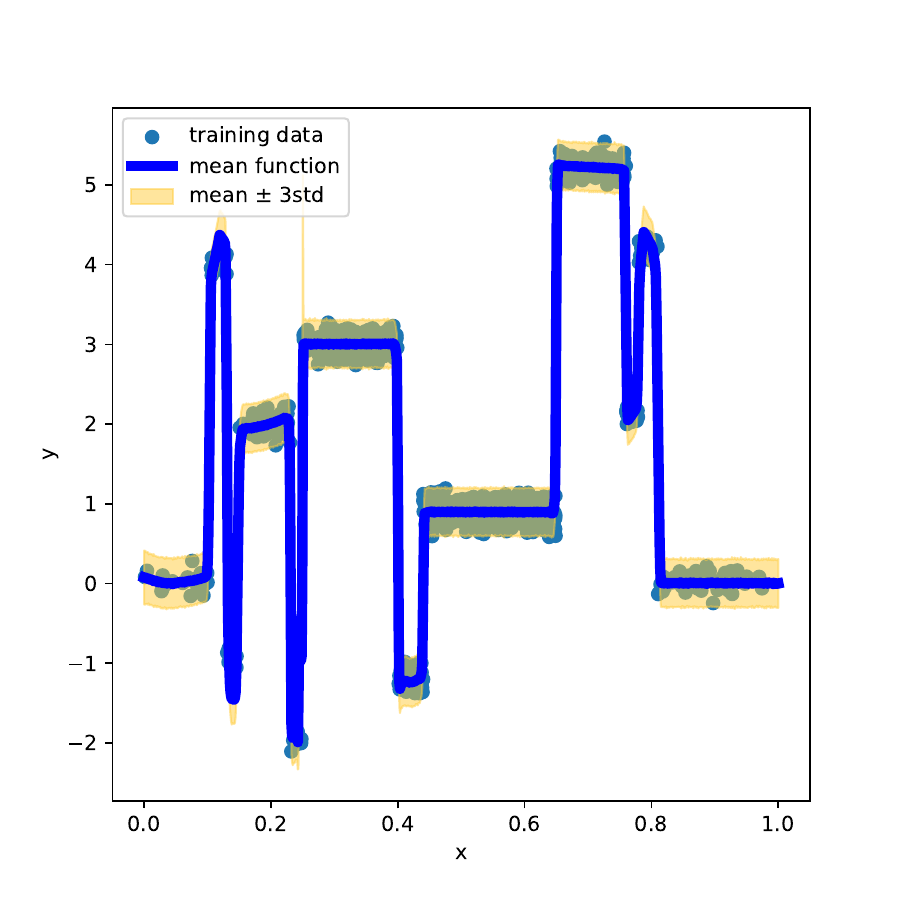}
        \includegraphics[width=0.32\textwidth]{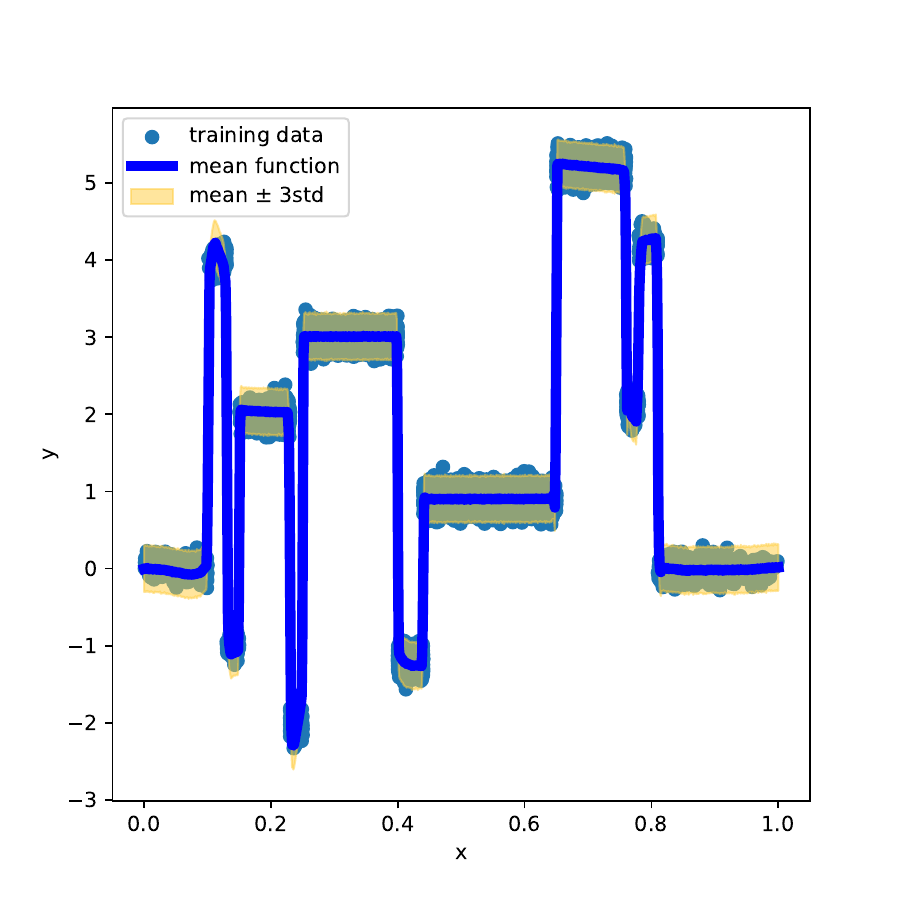}
        \caption{The results of estimating $f_3$.}
    \end{subfigure}
    \begin{subfigure}[b]{0.95\textwidth}
        \includegraphics[width=0.32\textwidth]{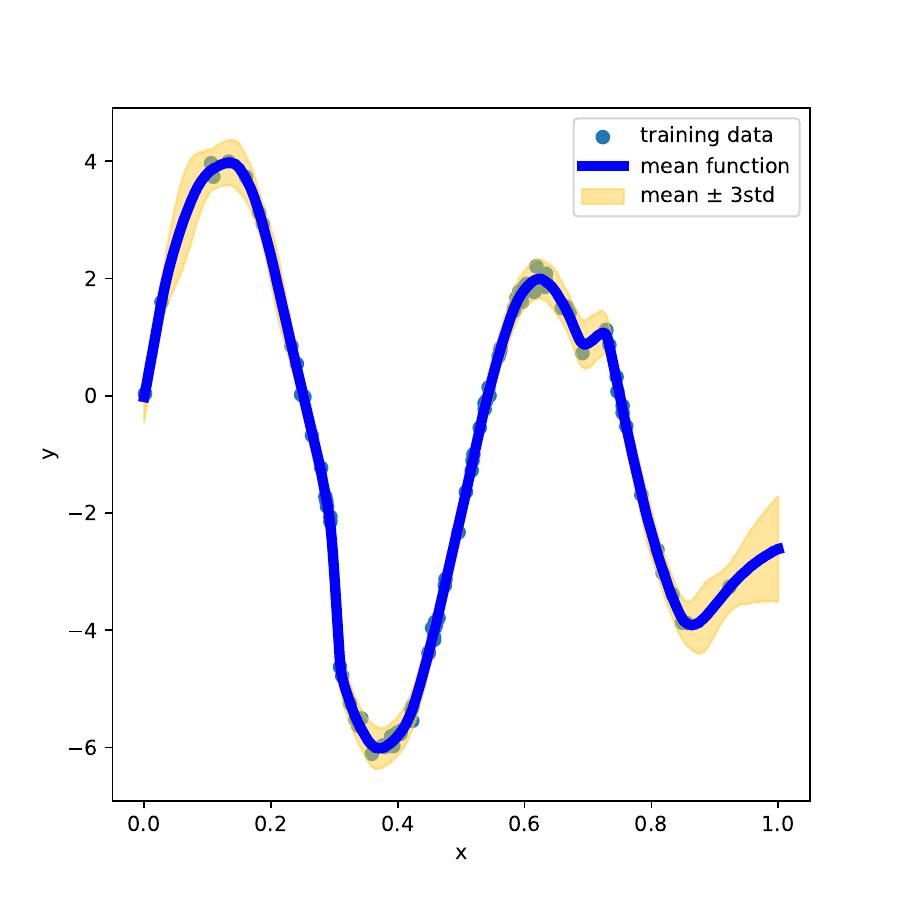}
        \includegraphics[width=0.32\textwidth]{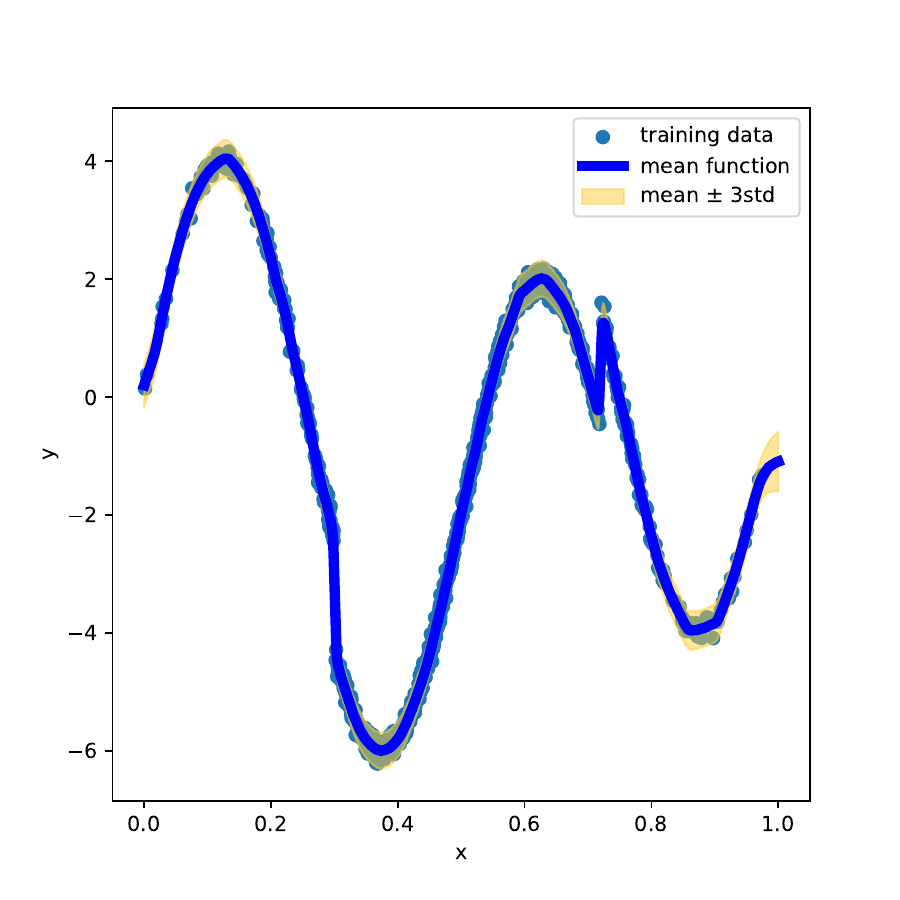}
        \includegraphics[width=0.32\textwidth]{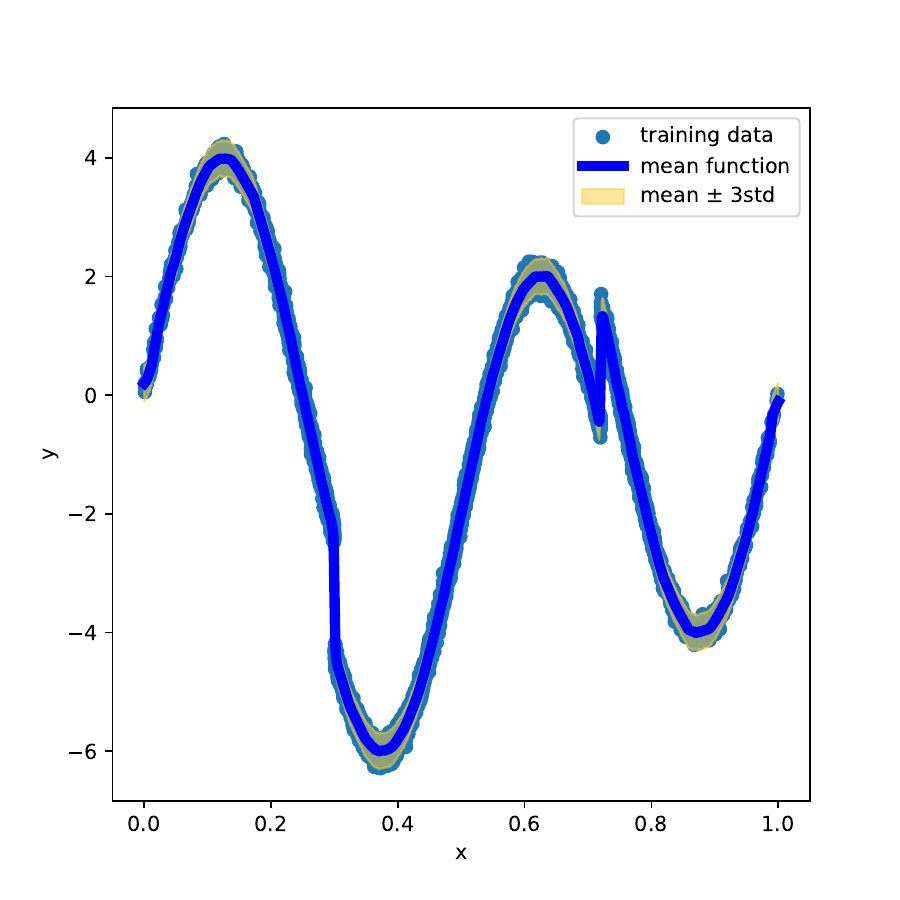}
        \caption{The results of estimating $f_4$.}
    \end{subfigure}
\end{figure}

\subsection{Comparison}

We compare two model in terms of empirical errors (training error) $\frac{1}{n}\sum_{i=1}^n \abs{y_i - f(x_i)}^2$ and $L^2$ norms (test error) $e_f = \norm{f - f_0}_{L^2}$ between true function $f_0$ and sampled function $f$. We approximate $e_f$ as $\hat{e}_f$ using Riemann sum,
\begin{equation*}
    \hat{e}_f = \left( \frac{1}{n^\ast} \sum_{k=1}^{n^\ast} (f(x_k^\ast) - f_0(x_k^\ast))^2  \right)^{1/2}, \quad x_k^* = \frac{k}{n^\ast},~ 1\leq k \leq n^\ast.
\end{equation*}
As in \autoref{table:EE} and \autoref{table:L2}, both models give similar results. For the functions $f_1$ and $f_2$, Gaussian mixture BNN shows better results for both errors. For the functions $f_3$ and $f_4$, Gaussian BNN gives slightly better results for the test error. However, as shown in \autoref{fig:function_2GMM_results} and \autoref{fig:function_Gaussian_results}, the differences are negligible.

\begin{table}[ht]
    \centering
    \caption{Summary statistics of the empirical errors between true functions and sampled functions, where the numbers in the parentheses denote the standard deviations.}
    \label{table:EE}
    \begin{tabular}{lrrr}
    \toprule
    $f_0$ & $n$ & Gaussian mixture BNN & Gaussian BNN \\
    \midrule
    \multirow[t]{3}{*}{$f_1$} & 100 & $1.398\times10^{-2}$ ($8.896\times10^{-4}$) & \color{blue} $1.386\times10^{-2}$ ($8.774\times10^{-4}$) \\
     & 1,000 & \color{blue} $1.451\times10^{-2}$ ($2.899\times10^{-4}$) & $1.455\times10^{-2}$ ($2.840\times10^{-4}$) \\
     & 10,000 & \color{blue} $1.455\times10^{-2}$ ($9.233\times10^{-5}$) & $1.461\times10^{-2}$ ($8.763\times10^{-5}$) \\
    \multirow[t]{3}{*}{$f_2$} & 100 & \color{blue} $1.336\times10^{-2}$ ($8.119\times10^{-4}$) & $1.359\times10^{-2}$ ($9.023\times10^{-4}$) \\
     & 1,000 & \color{blue} $1.424\times10^{-2}$ ($2.781\times10^{-4}$) & $1.427\times10^{-2}$ ($2.741\times10^{-4}$) \\
     & 10,000 & $1.420\times10^{-2}$ ($8.728\times10^{-5}$) & \color{blue} $1.419\times10^{-2}$ ($8.651\times10^{-5}$) \\
    \multirow[t]{3}{*}{$f_3$} & 100 & \color{blue} $1.428\times10^{-1}$ ($9.301\times10^{-3}$) & $1.497\times10^{-1}$ ($9.356\times10^{-3}$) \\
     & 1,000 & $2.297\times10^{-1}$ ($4.256\times10^{-3}$) & \color{blue} $2.289\times10^{-1}$ ($4.632\times10^{-3}$) \\
     & 10,000 & \color{blue} $2.200\times10^{-1}$ ($3.228\times10^{-3}$) & $2.204\times10^{-1}$ ($2.668\times10^{-3}$) \\
    \multirow[t]{3}{*}{$f_4$} & 100 & \color{blue} $1.389\times10^{-1}$ ($9.279\times10^{-3}$) & $1.390\times10^{-1}$ ($9.085\times10^{-3}$) \\
     & 1,000 & \color{blue} $1.430\times10^{-1}$ ($2.898\times10^{-3}$) & $1.460\times10^{-1}$ ($2.921\times10^{-3}$) \\
     & 10,000 & $1.543\times10^{-1}$ ($8.986\times10^{-4}$) & \color{blue} $1.500\times10^{-1}$ ($8.825\times10^{-4}$) \\
    \bottomrule
    \end{tabular}
\end{table}

\begin{table}[ht]
    \centering
    \caption{Summary statistics of the test errors between true functions and sampled functions, where the numbers in the parentheses denotes the standard deviations.}
    \label{table:L2}
    \begin{tabular}{lrrr}
    \toprule
    $f_0$ & $n$ & Gaussian mixture BNN & Gaussian BNN \\
    \midrule
    \multirow[t]{3}{*}{$f_1$} & 100 & \color{blue} $1.921\times10^{-2}$ ($1.441\times10^{-3}$) & $2.359\times10^{-2}$ ($3.646\times10^{-3}$) \\
     & 1,000 & \color{blue} $1.327\times10^{-2}$ ($3.172\times10^{-4}$) & $1.335\times10^{-2}$ ($5.983\times10^{-4}$) \\
     & 10,000 & $1.212\times10^{-2}$ ($2.808\times10^{-4}$) & \color{blue} $1.201\times10^{-2}$ ($2.643\times10^{-4}$) \\
    \multirow[t]{3}{*}{$f_2$} & 100 & \color{blue} $1.363\times10^{-2}$ ($6.582\times10^{-4}$) & $1.652\times10^{-2}$ ($2.293\times10^{-3}$) \\
     & 1,000 & \color{blue} $1.151\times10^{-2}$ ($3.835\times10^{-4}$) & $1.164\times10^{-2}$ ($4.203\times10^{-4}$) \\
     & 10,000 & $1.177\times10^{-2}$ ($2.626\times10^{-4}$) & \color{blue} $1.155\times10^{-2}$ ($2.529\times10^{-4}$) \\
    \multirow[t]{3}{*}{$f_3$} & 100 & $1.120\times10^{0}$ ($5.175\times10^{-2}$) & \color{blue} $1.041\times10^{0}$ ($6.923\times10^{-2}$) \\
     & 1,000 & $3.750\times10^{-1}$ ($5.497\times10^{-3}$) & \color{blue} $3.699\times10^{-1}$ ($6.699\times10^{-3}$) \\
     & 10,000 & \color{blue} $2.478\times10^{-1}$ ($4.589\times10^{-3}$) & $2.479\times10^{-1}$ ($4.036\times10^{-3}$) \\
    \multirow[t]{3}{*}{$f_4$} & 100 & $4.942\times10^{-1}$ ($2.930\times10^{-2}$) & \color{blue} $4.493\times10^{-1}$ ($4.450\times10^{-2}$) \\
     & 1,000 & $1.715\times10^{-1}$ ($8.959\times10^{-3}$) & \color{blue} $1.666\times10^{-1}$ ($1.029\times10^{-2}$) \\
     & 10,000 & $1.550\times10^{-1}$ ($3.242\times10^{-3}$) & \color{blue} $1.128\times10^{-1}$ ($2.454\times10^{-3}$) \\
    \bottomrule
    \end{tabular}
\end{table}

\end{document}